\documentclass[pdflatex,sn-mathphys-num,iicol]{sn-jnl}

\usepackage{graphicx}%
\usepackage{subfig}
\usepackage{multirow}%
\usepackage{amsmath,amssymb,amsfonts}%
\usepackage{amsthm}%
\usepackage{mathrsfs}%
\usepackage[title]{appendix}%
\usepackage{xcolor}%
\usepackage{textcomp}%
\usepackage{manyfoot}%
\usepackage{booktabs}%
\usepackage{algorithm}%
\usepackage{algorithmicx}%
\usepackage{algpseudocode}%
\usepackage{listings}%
\usepackage{geometry}%调整页边距
\usepackage{lineno, hyperref}
\usepackage[nameinlink,capitalize]{cleveref}
\usepackage{booktabs}
\newtheorem{theorem}{Theorem}%  meant for continuous numbers
\newtheorem{remark}{Remark}%

\newtheorem{definition}{Definition}%
\newtheorem{assumption}{Assumption}

\begin{document}

\title[Article Title]{Towards Fair Class-wise Robustness: Class Optimal Distribution Adversarial Training}

\author[1]{\fnm{Hongxin} \sur{Zhi}}\email{newzhx@yeah.net}
\author*[1]{\fnm{Hongtao} \sur{Yu}}\email{yht202209@163.com}
\author[1]{\fnm{Shaomei} \sur{Li}}\email{lishaomei\_may@126.com}
\author[1]{\fnm{Xiuming} \sur{Zhao}}\email{zhaoxiuming@henu.edu.cn}
\author[1]{\fnm{Yiteng} \sur{Wu}}\email{wuyiteng1992@163.com}

\affil[1]{\orgdiv{Institute of Information Technology}, \orgname{Information Engineering University}, \orgaddress{\city{Zhengzhou}, \postcode{450000}, \state{Henan}, \country{China}}}


\abstract{%要研究的问题
Adversarial training has proven to be a highly effective method for improving the robustness of deep neural networks against adversarial attacks. Nonetheless, it has been observed to exhibit a limitation in terms of robust fairness, characterized by a significant disparity in robustness across different classes.
%现有方法存在的不足
Recent efforts to mitigate this problem have turned to class-wise reweighted methods. However, these methods suffer from a lack of rigorous theoretical analysis and are limited in their exploration of the weight space, as they mainly rely on existing heuristic algorithms or intuition to compute weights.
In addition, these methods fail to guarantee the consistency of the optimization direction due to the decoupled optimization of weights and the model parameters. They potentially lead to suboptimal weight assignments and consequently, a suboptimal model.
%提出的框架
To address these problems, this paper proposes a novel min-max training framework, Class Optimal Distribution Adversarial Training (CODAT), which employs distributionally robust optimization to fully explore the class-wise weight space, thus enabling the identification of the optimal weight with theoretical guarantees.
%提出的闭合解与确定性目标函数
Furthermore, we derive a closed-form optimal solution to the internal maximization and then get a deterministic equivalent objective function, which provides a theoretical basis for the joint optimization of weights and model parameters.
%弹性系数
Meanwhile, we propose a fairness elasticity coefficient for the evaluation of the algorithm with regard to both robustness and robust fairness.
%实验结果
Experimental results on various datasets show that the proposed method can effectively improve the robust fairness of the model and outperform the state-of-the-art approaches.}


\maketitle

\section{Introduction}\label{sec:1}

Deep Neural Networks (DNNs) have made remarkable breakthroughs in a variety of domains, demonstrating tremendous potential and power. However, along with the rapid advancement of this technology, a significant security challenge has emerged: the problem of adversarial vulnerability~\cite{2014arXiv1412.6572G,2013arXiv1312.6199S}. It has been shown that by adding small perturbations to natural inputs, it is easy to fool otherwise accurate models into making incorrect predictions. As DNNs become more deeply and broadly integrated into everyday life, especially in security-critical areas~\cite{Chen_Seff_Kornhauser_Xiao_2015,Ma_Niu_Gu_Wang_Zhao_Bailey_Lu_2021,2019arXiv190700374M,Sharif_Bhagavatula_Bauer_Reiter_2016}, the potential risks posed by adversarial vulnerabilities have increased dramatically. In response to these risks, a number of defense methods~\cite{2017arXiv170606083M,Papernot_McDaniel_Wu_Jha_Swami_2016,Raghunathan_Steinhardt_Liang_2018,Xie_Wu_Maaten_Yuille_He_2019,Xu_Evans_Qi_2018,zhi2024ma,Zhang_Yu_Jiao_Xing_Ghaoui_Jordan_2019} have been proposed. Among these, Adversarial Training (AT)~\cite{2017arXiv170606083M,Zhang_Yu_Jiao_Xing_Ghaoui_Jordan_2019}, is considered to be one of the most effective defense mechanisms and has received considerable attention in the robust machine learning community.
%===================================Figure 1=========================================
\begin{figure}
    \centering
    \subfloat[Class-wise robust accuracy disparity]{
    \includegraphics[width=0.42\linewidth]{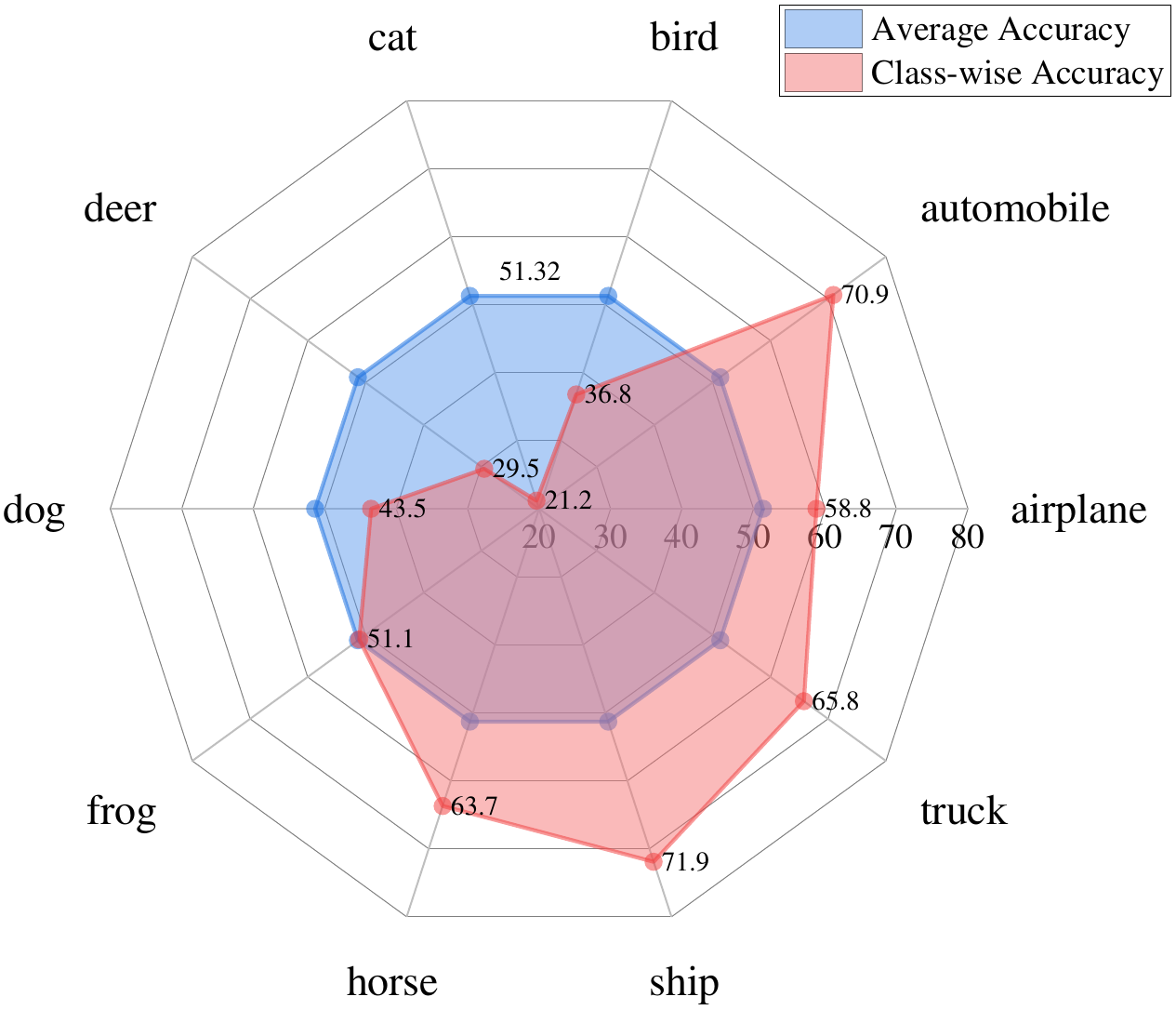}
    \label{fig:1-a}
    }
    \hfill
    \subfloat[Class-wise natural accuracy disparity]{
    \includegraphics[width=0.48\linewidth]{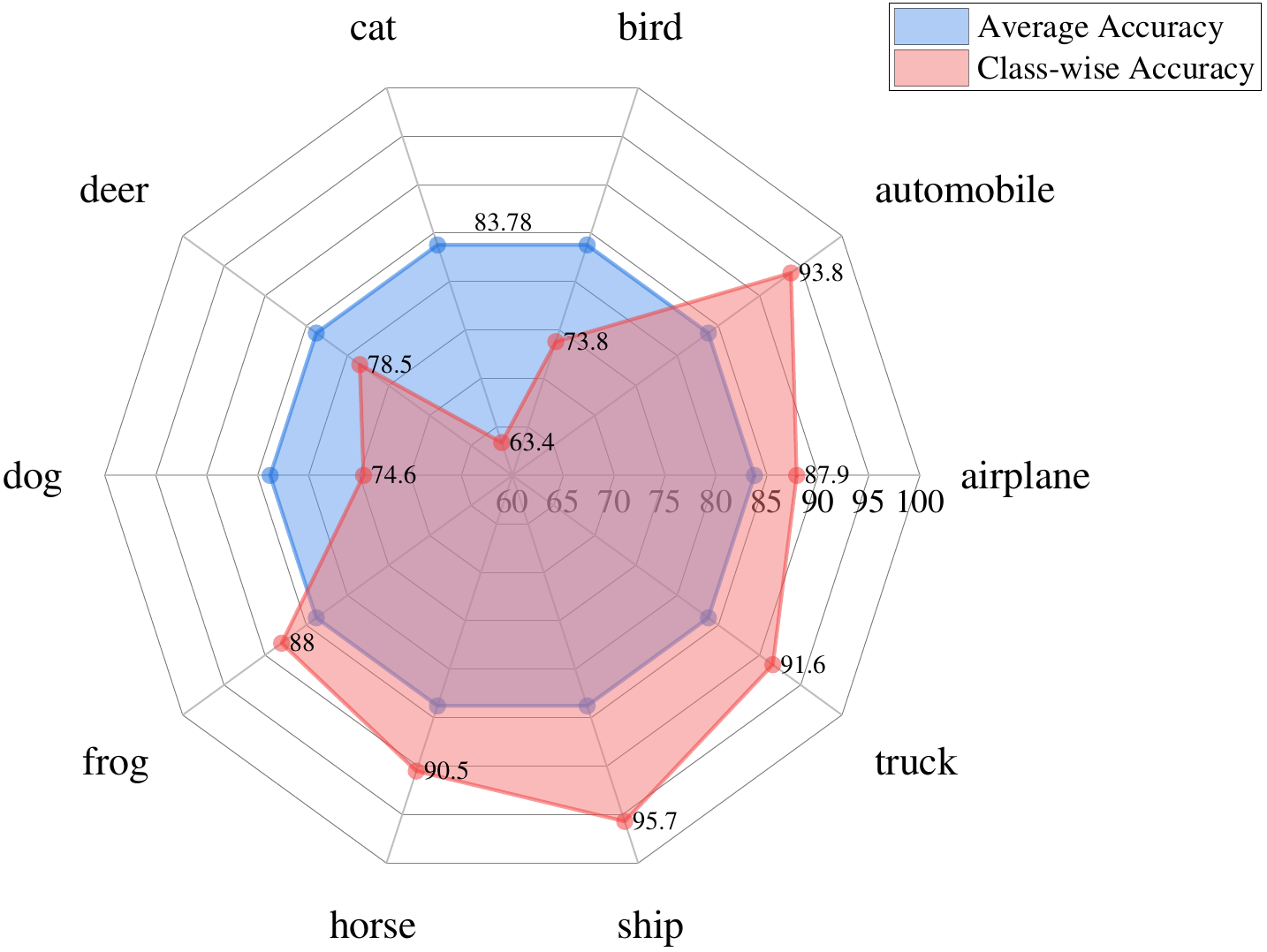}
    \label{fig:1-b}
    }
    \caption{Diagram of robust fairness problem in an adversarially trained model on CIFAR-10 using ResNet-18 under $\ell_\infty$ threat model (with the 10-step $\ell_\infty$ PGD attack). The model exhibits inconsistency in terms of accuracy on both adversarial examples (generated by the 20-step $\ell_\infty$ PGD attack) and natural input.}
    \label{fig:1}
\end{figure}
%=======================================================================================

Although promising in improving the robustness of models against adversarial attacks, adversarial training still faces considerable limitations. Recent studies~\cite{benz2021robustness,2021arXiv210514240T,xu2021robust} have highlighted a critical issue: adversarially trained models can exhibit severe disparity in robust accuracy across classes, even when trained on well-balanced datasets, which has been referred to as the problem of robust fairness~\cite{xu2021robust}. For example, on the CIFAR-10 dataset, a robust ResNet-18 model has a robust accuracy of 71.9\% for class \textit{ship}, whereas it only achieves 21.2\% for class \textit{cat}, resulting in a significant gap of 50.7\%, as shown in~\cref{fig:1-a}. This class-wise disparity is also evident in the model’s accuracy on natural data (see~\cref{fig:1-b}). Skilled adversaries can identify these more vulnerable classes and launch targeted adversarial attacks, thereby increasing the likelihood of successful attacks. So, this problem greatly amplifies the potential security risks associated with adversarial vulnerability, especially in security-critical domains.

%====================================Figure 2===========================================================
\begin{figure*}[htp]
    \centering
    \includegraphics[width=\textwidth]{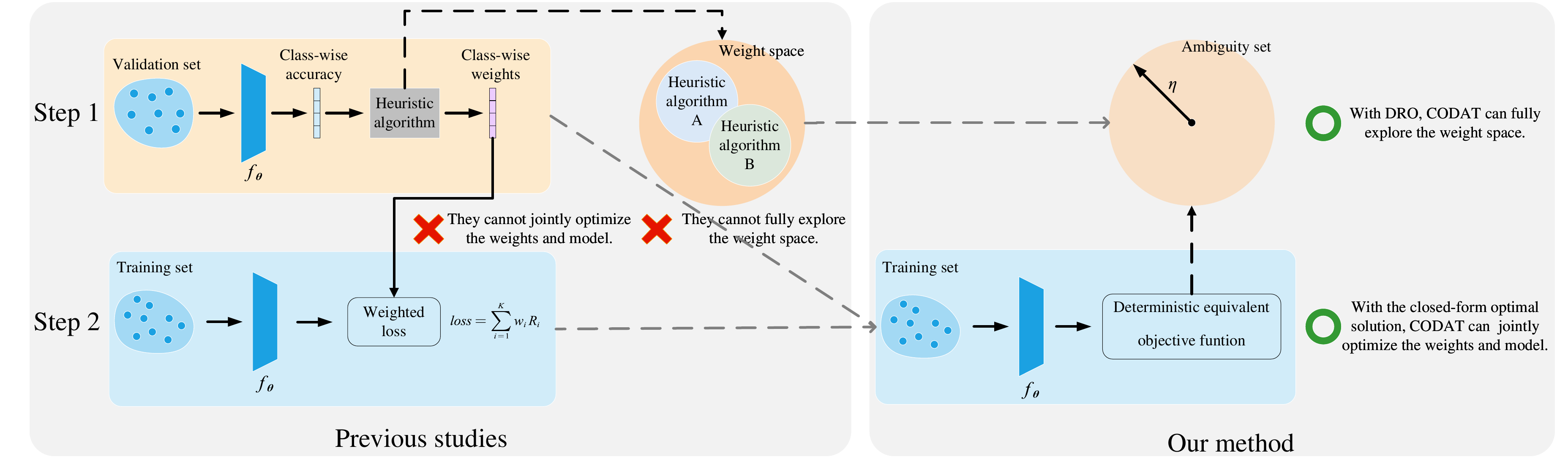}
    \caption{A simple comparison of previous studies and our method. Previous studies, which rely on heuristic algorithms for weight computation, are constrained by two limitations: an inability to fully explore the weight space and a lack of capability to jointly optimize weights and model parameters. Our method, informed by Distributionally Robust Optimization principles and equipped with a closed-form optimal solution, effectively surmounts these limitations.}
    \label{fig:2}
\end{figure*}
%=====================================================================================================

Some studies~\cite{benz2021robustness,xu2021robust,li2023wat,2023arXiv230208872P,sun2023improving,10205260} have been devoted to addressing this problem and have made some substantial progress. Although the approaches vary, they generally fall into the class reweighting method, which assigns an important weight to each class based on a specific algorithm to make more rational use of the limited model capacity. In particular, by increasing the weights of classes with poorer performance, these methods make the model pay more attention to the classes that play a critical role in determining the decision boundaries. However, these methods mainly rely on specific heuristic algorithms or intuition to determine the class-wise weights, which lack rigorous theoretical analysis and reliability and may be influenced by subjective and stochastic factors in the selection of algorithms, as shown in the left diagram of~\cref{fig:2}. This situation limits their capacity to explore the class-wise weight space. Furthermore, the goal of these methods is to optimize the model parameters with an existing heuristic algorithm or intuition, where the procedures for weight and model optimization are decoupled, i.e., in two separate steps. In such circumstances, it might be challenging for them to guarantee the consistency of the optimization direction with respect to weight and model parameters. Consequently, the weight assignments made during the training phase may be suboptimal, which often results in an overall suboptimal model. Therefore, the following question is naturally raised:

\begin{quote}
    \textit{How can we systematically explore the weight space under the guidance of a theoretical framework, while also guaranteeing the consistency of weight and model optimization direction to achieve overall optimality?}
\end{quote}

In this paper, we address this important yet overlooked issue in the domain of robust fairness. To fully explore the weight space, we propose a novel training framework termed Class Optimal Distribution Adversarial Training (CODAT), inspired by the principles of Distributionally Robust Optimization (DRO)~\cite{2019arXiv190805659R}. Similar to standard adversarial training, CODAT also can be formalized as a min-max optimization problem. However, its internal maximization objective is to maximize the class-wise expected risk, which is the core idea of DRO. While the external minimization task is to train the model by minimizing this maximum risk. With the guidance of DRO theory, CODAT can fully explore the class adversarial distribution space, thereby enabling the learning of the worst-case class adversarial distribution that maximizes the class-wise expected risk\footnote{In this paper, the class adversarial distribution is utilized as class weights, given their equivalence in the process of calculating the class-wise expected risk.} with theoretical guarantees, as depicted in the right diagram of~\cref{fig:2}. It should be noted that this maximum risk represents the upper bound on the risk under any possible class distribution. To reach this upper bound, CODAT needs to assign greater weights to those classes with higher risk, thus making the model more attentive to these harder-to-train classes. Consequently, the worst-case distribution represents the optimal weight assignments.

Furthermore, to ensure that the optimization directions of the weights and model are consistent, we first derive a closed-form optimal solution to the internal maximization problem and then get a deterministic equivalent objective function that combines this closed-form solution within the original objective of the model. This equivalent objective function, therefore, provides the theoretical support for the joint optimization of weights and the model. Moreover, due to their trade-off~\cite{benz2021robustness,li2023wat}, we propose the fairness elasticity coefficient as a measure to evaluate both the robustness and robust fairness of an algorithm.

The contributions of this paper can be summarized as follows:
\begin{itemize}
    \item We propose a novel min-max training framework, called Class Optimal Distribution Adversarial Training (CODAT), which fully explores the class adversarial distribution space with the guidance of distributionally robust optimization theory to find the optimal (worst-case) distribution.
    \item We derive a closed-form optimal solution to the internal maximization problem. Upon the solution, we obtain a deterministic equivalent objective function that integrates the optimal solution into the model's original objective. This integration enables the joint optimization and aligns the optimization directions of the weights and the model.
    \item We propose the fairness elasticity coefficient as a measure of evaluating an algorithm in terms of both robustness and robust fairness.
    \item Extensive experiments on various benchmark datasets demonstrate that our CODAT outperforms state-of-the-art methods.
\end{itemize}

\section{Related work}\label{sec:2}
\subsection{Adversarial training}\label{sec:2-1}
Adversarial training~\cite{2017arXiv170606083M,Zhang_Yu_Jiao_Xing_Ghaoui_Jordan_2019} is widely regarded as one of the most effective methods of defense against adversarial attacks. It can be formulated as a min-max optimization problem~\cite{2017arXiv170606083M}, where the inner maximization problem aims to find the worst-case adversarial counterpart corresponding to each natural example, and the outer minimization problem then utilizes these adversarial examples to optimize the model's parameters. Despite its success, adversarial training presents some challenges. Numerous studies~\cite{Zhang_Yu_Jiao_Xing_Ghaoui_Jordan_2019,Tsipras_Santurkar_Engstrom_Turner_Madry_2018,Yang_Rashtchian_Zhang_Salakhutdinov_Chaudhuri_2020} have confirmed the existence of a trade-off between model robustness and natural accuracy in the adversarial training paradigm. TRADES, which has become a benchmark algorithm, effectively mitigates this problem by decomposing the loss into robust loss and natural loss, and controlling the trade-off between the two through a hyperparameter. Furthermore, Rice et. al.~\cite{2020arXiv200211569R} revealed the phenomenon of robust overfitting in adversarial training, which is characterized by a continual reduction in training loss as the learning rate decays, while the test loss initially drops before escalating rapidly. In response, they proposed an "early stopping " mechanism to cope with it. In parallel with these studies, this paper focuses on the recently identified robust fairness problem within the domain of adversarial training.

\subsection{Robust fairness}\label{sec:2-2}
The class-wise disparity problem is a typical issue within the domain of long-tailed recognition. However, recent research~\cite{benz2021robustness,2021arXiv210514240T,xu2021robust} has found that a similar problem occurs with adversarial training on balanced datasets. To address this issue, Benz et al.~\cite{benz2021robustness} employ a cost-sensitive learning approach, which is widely used in long-tailed problems, to assign weights to each class based on their respective robust accuracy. FRL-RWRM~\cite{xu2021robust} fine-tunes a robust model through the re-weight and re-margin strategy to improve the robust fairness of the model. In contrast to FRL-RWRM, CFA~\cite{10205260} trains a model from scratch by dynamically adjusting the class-wise perturbation margin and regularization based on the class-wise robust accuracy on a validation set. WAT~\cite{li2023wat} adapts the framework of standard AT to focus on the worst-class during the training phase and draws on the Hedge algorithm~\cite{Freund_Schapire_1997} to compute the weights. However, the majority of these methods depend on heuristic algorithms or intuition to assign the class-wise weights, which lacks the theoretical guarantee necessary for a comprehensive exploration of the weight space. Moreover, the utilization of existing heuristic algorithms or intuition in these methods inherently separates the process of weight computation from model optimization. This separation cannot ensure the consistency of optimization objectives between weight and model parameters. In this paper, we aim to identify class-wise weights through a comprehensive exploration of the weight space with theoretical guarantees and to achieve the joint optimization of the class-wise weight and model parameters.

\subsection{Distributionally robust optimization}\label{sec:2-3}
In practical optimization scenarios, we often have access to only a limited sample of data, which allows us to estimate the empirical distribution but not the true underlying distribution. This uncertainty can lead to a degradation in model performance when dealing with unseen data. To address this challenge, DRO assumes that the true distribution lies in the vicinity of the empirical distribution and aims to identify the worst-case distribution within an ambiguous set near this empirical distribution. The model is then optimized based on this worst-case scenario to enhance its robustness against the true distribution. Selecting an appropriate ambiguity set is a crucial step in the DRO problem, which helps to accurately identify the worst-case distribution, leading to more effective model optimization. One widely used approach to constructing the ambiguity set is to specify that the distribution should reside within a certain threshold distance from the empirical distribution. In this paper, we utilize the $\chi^2$-divergence to measure the distance between two distributions.

Among contemporary studies, CFOL~\cite{2023arXiv230208872P} and FAAL~\cite{zhang2024towards} are perhaps the most analogous to our work, as they also utilize distributionally robust optimization framework to adress the challenge of robust fairness in adversarial training. However, there are a number of fundamental differences that distinguish CFOL, FAAL and CODAT. Firstly, with regard to the ambiguity set: CFOL establishes a direct link with the Conditional Value at Risk (CVaR)~\cite{Rockafellar_Uryasev_2016} ambiguity set, while FAAL relies on Kullback-Leibler (KL) divergence. In contrast, CODAT employs an ambiguity set based on the $\chi^2$-divergence. This choice of divergence metric is a crucial distinction, as it affects the robustness and applicability of the models. Secondly, the approach to the inner maximization problem diverges among the three. CFOL employs heuristic algorithms, including the Hedge algorithm~\cite{Freund_Schapire_1997} and the Exp3 algorithm~\cite{Auer_Cesa-Bianchi_Freund_Schapire_2002}, to approximate the inner maximization problem.  In contrast, FAAL employs conic convex optimization method to address the problem. In stark contrast, we derive a closed-form optimal solution, yielding a deterministic equivalent objective function. This approach ensures a more direct and theoretically grounded method for identifying the worst-case adversarial distribution. Additionally, there are differences in the learning paradigms adopted. CFOL operates as an online learning algorithm, whereas both FAAL and CODAT employ a batch learning approach. 

\section{Preliminary and Problem analysis}\label{sec:3}
In this section, we begin by introducing the notions and concepts utilized throughout this paper. We then proceed to review the principles of standard adversarial training, establishing a conceptual framework for our subsequent discussion. Our empirical analysis follows, with a focus on elucidating the underlying causes of the robust fairness challenge. Concluding this section, we present insights into current methods aimed at addressing this problem, including an analysis of their potential limitations.
%=========================================================================================================
\subsection{Notation}\label{sec:3-1}
In this paper, we focus on a $K$-class robust classification task over the input space $\mathcal{X}\subset \mathbb{R}^d$ and the output space $\mathcal{Y}=\{1,\cdots,K\}=:[K]$. Let $\mathbb{D}$ denote a distribution over $\mathcal{X}\times \mathcal{Y}$. $\mathcal{S}=\{(\boldsymbol{x}_i,y_i)\}_{i=1}^n$ denotes a training set sampled from $\mathcal{D}$, where $\boldsymbol{x}_i\in\mathcal{X}$, $y_i\in\mathcal{Y}$, and $n$ is the total number of examples. Assume $\mathcal{F}$ represents the hypothesis class, while $f_{\boldsymbol{\theta}}:\mathcal{X}\rightarrow\mathcal{Y}$ is a DNN classifier within $\mathcal{F}$, parameterized by $\boldsymbol{\theta}$. Let $l:\mathbb{R}^K\times\mathcal{Y}\rightarrow\mathbb{R}$ be the loss function, which is typically the cross-entropy loss. $||\cdot||_p$ denotes the $\ell_p$-norm. Let $\mathcal{B}(\boldsymbol{x},\epsilon)=\{\boldsymbol{x}':||\boldsymbol{x}'-\boldsymbol{x}||_p\leq \epsilon\}$ be the $\ell_p$-norm ball centered at $\boldsymbol{x}$ with radius $\epsilon>0$, which specifies the constraints of the perturbed examples. Within the scope of this paper, we adopt the $\ell_\infty$ threat model.

\subsection{Standard adversarial training}\label{sec:3-2}
The standard adversarial training on $\mathcal{S}=\{(\boldsymbol{x}_i,y_i)\}_{i=1}^n$ with DNN classifier $f_{\boldsymbol{\theta}}$ can be fourmulated as the following min-max optimization problem~\cite{2017arXiv170606083M}:

\begin{equation}\label{eq:1}
\min_{\boldsymbol{\theta}}\frac{1}{n}\sum\limits_{i=1}^n\max\limits_{\boldsymbol{x}'_i\in\mathcal{B}(\boldsymbol{x}_i,\epsilon)}l(f_{\boldsymbol{\theta}}(\boldsymbol{x}'_i),y_i),
\end{equation}
where $\boldsymbol{x}'_i$ is the adversarial counterpart of natural example $\boldsymbol{x}_i$, which is within the norm ball $\mathcal{B}(\boldsymbol{x}_i,\epsilon)$. By sequentially addressing the internal maximization and external minimization problems, standard AT enhances the robustness of the model.
\subsection{The potential reason for robust fairness problem in standard AT}\label{sec:3-3}
%=============================================================================================
\begin{figure*}
    \centering
    \subfloat[The natural examples distribution]{
    \includegraphics[scale=.28]{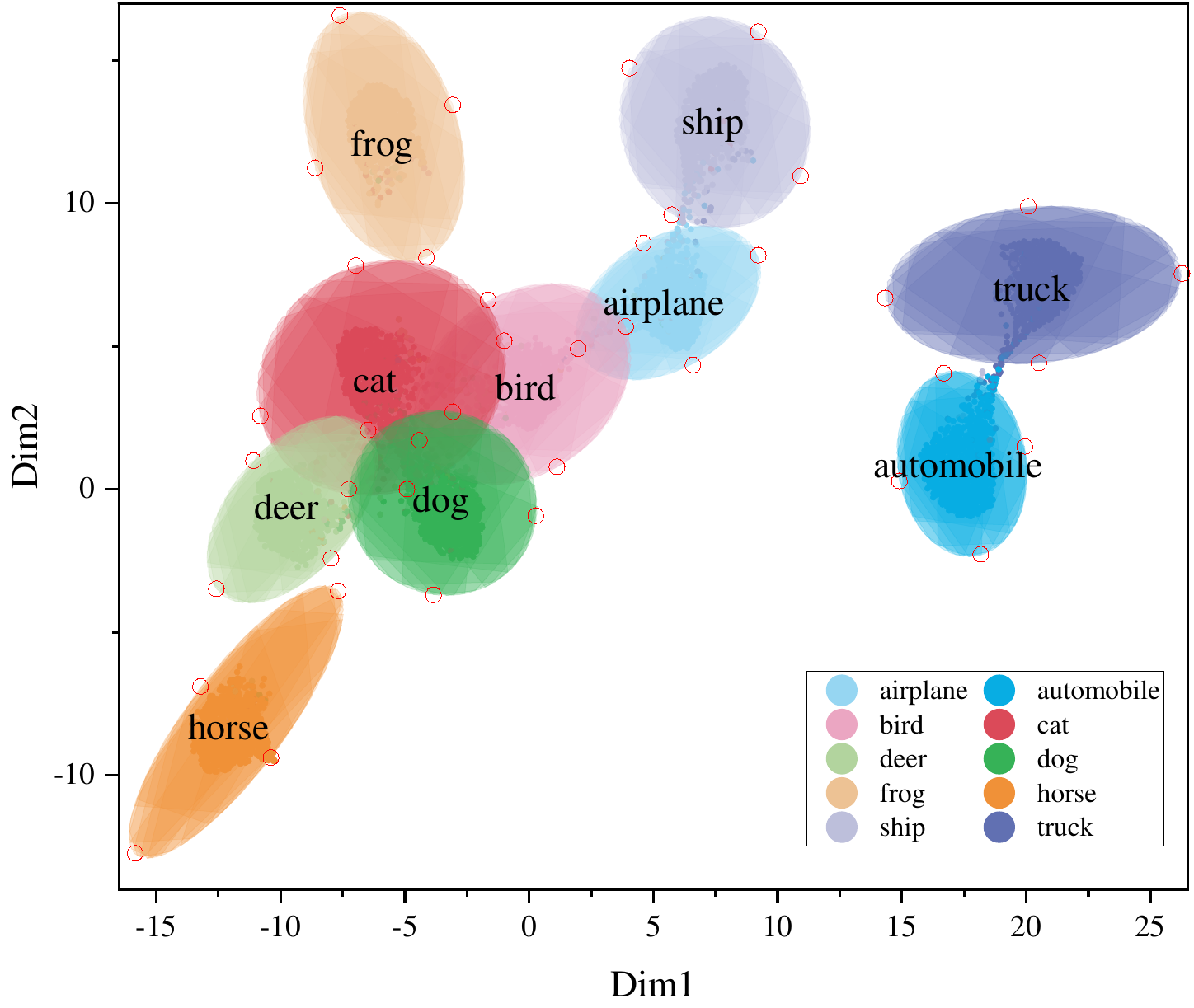}
    \label{fig:3-1}
    }
    \hfill
    \subfloat[Confusion matrix of robustness]{
    \includegraphics[scale=.28]{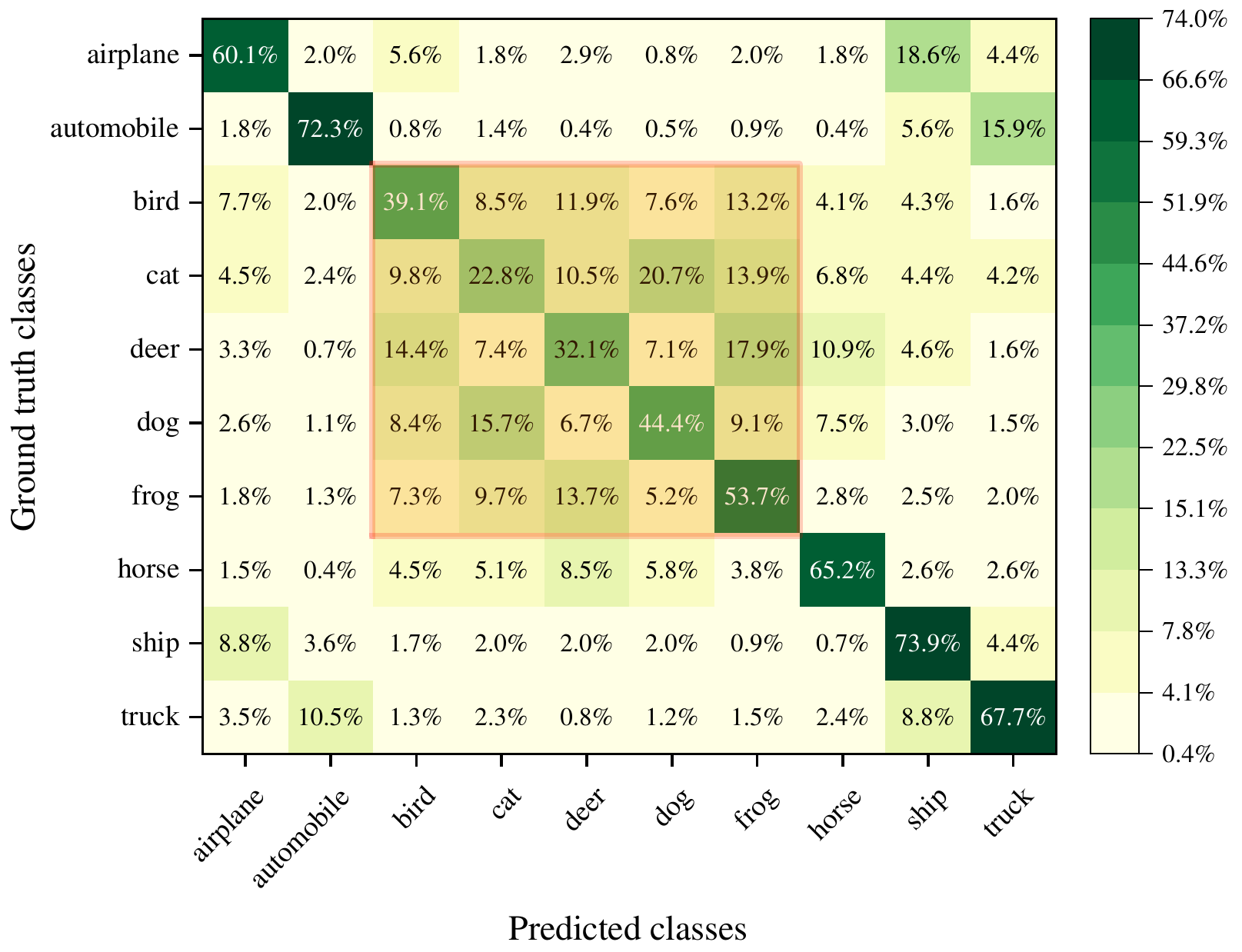}
    \label{fig:3-2}
    }
    \caption{Non-uniform semantic distances among classes in the CIFAR-10 test set. (a) UMAP visualization of the distribution of natural examples with a naturally trained ResNet-18. (b) The confusion matrix of robustness under PGD-20 attack with a robust ResNet-18. The red-highlighted areas within the matrix indicate classes that are more frequently misclassified by the model. It can be observed that classes that are more semantically closer tend to exhibit higher misclassification rates.}
    \label{fig:3}
\end{figure*}
%=====================================================================================================
We argue that the classes within a dataset, as delineated by human cognition, may not exhibit uniform semantic distances between every pair, potentially leading to the problem of robust fairness. \cref{fig:3-1} depicts the output distribution of the penultimate layer’s representation from a naturally trained ResNet-18 model when evaluated on the CIFAR-10 test set, utilizing the UMAP (Uniform Manifold Approximation and Projection) technique~\cite{2018arXiv180203426M} for dimensionality reduction. The visualization clearly shows that the class \textit{cat} exhibits a notably closer semantic proximity to classes such as \textit{dog}, \textit{bird}, and \textit{deer}, in contrast to its greater semantic distance from classes such as \textit{automobile} and \textit{truck}. This distinction increases the potential for misclassification among the former set relative to the latter, especially in the presence of adversarial perturbations. 

In~\cref{fig:3-2}, we present the confusion matrix that assesses the robustness of a standard adversarially trained ResNet-18 model against the PGD-20 attack on the CIFAR-10 test set. It is important to note that since the attack is untargeted, the adversary will naturally seek out the category that is most vulnerable to attack. This behavior allows the confusion matrix to effectively reveal the inherent semantic distances among the classes. As decipted in~\cref{fig:3-2}, there is a significantly higher probability of misclassification among the classes \textit{cat}, \textit{dog}, \textit{bird}, \textit{deer}, and \textit{frog}\footnote{A trade-off exists between the capture of local and global structures in UMAP visualization, which can result in distortions when representing the global inter-class distances.}. When considered alongside~\cref{fig:3-1}, it is evident that classes with greater semantic similarity are also more frequently misclassified by the model.

Consequently, examples from different classes may exert an inequitable influence on the formation of decision boundaries during the training phase. It is crucial to differentiate between examples from different classes to ensure a more balanced and robust learning process. However, standard adversarial training addresses the external minimization problem by optimizing the average risk across all adversarial examples, treating examples from all classes equally. This approach fails to account for the inherent disparities between classes. Furthermore, DNNs are prone to a phenomenon termed shortcut learning~\cite{Geirhos_Jacobsen_Michaelis_Zemel_Brendel_Bethge_Wichmann_2020}, which inclines them to concentrate on learning from classes that are more readily classified. This tendency can lead to a lack of uniformity in the model's performance across different classes, as illustrated in~\cref{fig:1}.
%==============================================================================================
\subsection{Class-wise weighted adversarial training}\label{sec:3-4}
To encourage models to focus their attention on more challenging classes, numerous algorithms~\cite{benz2021robustness,sun2023improving,10205260} adopt a class-wise weighted adversarial training strategy. This method assigns weights to the robust risk of each class's examples based on a specific algorithmic criterion. Despite the diversity in the specific forms, these algorithms can be generally formulated as the following optimization problem:

\begin{equation}
    \label{eq:2}
    \begin{array}{cc}
        \min\limits_{\boldsymbol{\theta}}\sum\limits_{k=1}^Kw_kR_k^{\rm{rob}}:=\mathbb{E}_{(\boldsymbol{x},y)\in\mathcal{D}_k}[\max\limits_{\boldsymbol{x}'\in\mathcal{B}(\boldsymbol{x},\epsilon)}l(f_{\boldsymbol{\theta}}(\boldsymbol{x}'),y)],  \\
         {\rm{s.t.}}\quad \sum\limits_{k=1}^Kw_k=1, \\
         w_k\geq 0,
    \end{array}
\end{equation}
where $\mathcal{D}_k$ denotes the distribution of examples belonging to class $k$, $R_k^{\rm{rob}}$ represents the average robust risk for class $k$, and $w_k$ is the weight assigned to the class $k$.
\subsection{Worst-class adversarial training}
Unlike class-wise weighted adversarial training methods, WAT~\cite{li2023wat} and CFOL\footnote{In this paper, we consider the CFOL method to be a variant of class-weighted adversarial training approaches, given that its equivalence has been established in the original paper.}~\cite{2023arXiv230208872P}  are parallel to propose the worst-class adversarial training method that focuses on the worst class during the training phase. The objective can be expressed as follows:
\begin{equation}\label{eq:3}
	\min\limits_{\boldsymbol{\theta}}\max\limits_{k\in [K]}R_k^{\rm{rob}}.
\end{equation}

Both WAT and CFOL consider the problem posed by~\cref{eq:3} to be a zero-sum game, and they employ the Hedge algorithm for weight allocation.

These class-wise weighted and worst-class adversarial training algorithms assign greater weights to examples from classes that are more challenging to learn, with the goal of enhancing the model's robust fairness through the efficient utilization of limited model capacity. Nevertheless, these methods employ specific heuristic algorithms for weight determination, a process that is often separate from the model's optimization routine. This approach presents two primary challenges. Firstly, the heuristic-based weight computation process lacks robust theoretical underpinnings, which may limit the exploration of the weight space and result in suboptimal weight assignments. Secondly, the separation of the weight calculation from the model optimization process raises concerns about the consistency of the optimization objectives of the two problems. This inconsistency can result in a model that is not optimal.

\section{Method}\label{sec:4}
Building upon the aforementioned analysis, this section presents our proposed Class Optimal Distribution Adversarial Training (CODAT) framework. Furthermore, we present the deterministic equivalent objective function, which incorporates the derived closed-form optimal solution to the inner maximization problem.

\subsection{Class optimal distribution adversarial training}\label{sec:4-1}
To fully explore the weight space, we propose the class optimal distribution adversarial training framework, which draws inspiration from the principles of distributionally robust optimization. This framework models the class adversarial distribution and fully explores this distribution space to identify the class-wise weights in a principled manner, rather than relying on heuristic approaches. Specifically, the CODAT framework is designed to train the model to be robust against the worst-case class adversarial distribution, which corresponds to the optimal weight assignments. The framework can be formulated as the following min-max optimization problem:
\begin{equation}\label{eq:4}
    \begin{aligned}
    \min\limits_{\boldsymbol{\theta}}\max\limits_{P\in\mathbb{P}}\mathbb{E}_P[R_{\xi}^{\rm{rob}}], \\
        {\rm{s.t.}}\quad \sum\limits_{k=1}^Kp(\xi=k)=1,
    \end{aligned}
\end{equation}
where $\xi$ denotes a random variable representing the indices of the classes, taking values in the set $[K]$. $P$ represents the probatility distribution of $\xi$, specifically referring to the class adversarial distribution. The set $\mathbb{P}$ encompasses all possible distributions that $P$ can assume. It should be noted that the class-wise expected risk can be considered as the weighted sum of the risks of each class under a specified distribution $P$. Consequently, in~\cref{eq:4}, the distribution $P$ dictates the weights for each class, i.e., $p(\xi=k)$ represents the weights assigned to the samples of class $k$.

The inner maximization problem within the CODAT framework is designed to conduct a comprehensive exploration of the class adversarial distribution space, guided by theoretical guarantees, with the objective of identifying the optimal class adversarial distribution that maximizes the class-wise expected risk. Subsequently, the outer minimization problem concentrates on training the model by minimizing the expected risk under the identified worst-case distribution. It is important to note that the maximum expected risk effectively serves as an upper bound on the expected risk across all possible distributions. This characteristic ensures that, during the training process, classes which incur larger training losses are allocated higher weights. Consequently, the worst-case class adversarial distribution, as identified through this process, represents the optimal distribution for improving the model's robust fairness.

\begin{remark}\label{rmk:1}
    When $P$ is a Dirac distribution, the objective function in~\cref{eq:4} reduces to that in~\cref{eq:3}. Similarly, when $P$ is a uniform distribution, the objective function in~\cref{eq:4} reduces to the standard adversarial training objective presented in~\cref{eq:1}. Furthermore, \cref{eq:2} represents a specific instance within the scope of~\cref{eq:4}.
\end{remark}

Therefore, in comparison to existing methods, CODAT represents a more general framework that employs DRO theory to facilitate a comprehensive exploration of the class adversarial distribution space, thereby enabling the identification of optimal weight assignments. This approach overcomes the limitations of existing frameworks, which are typically based on specific heuristic algorithms for weight computation.

\subsection{A closed-form solution for the inner maximization problem}\label{sec:4-2}
To jointly optimize the class-wise weights and model parameters, we derive a closed-form solution for the inner maximization problem within the CODAT framework. This derivation yields a deterministic equivalent objective function, which in turn provides a theoretical foundation for the simultaneous optimization of both elements.

The CODAT framework, as presented in~\cref{eq:4}, can be reformulated as follows:
\begin{equation}\label{eq:5}
    \min\limits_{\boldsymbol{\theta}}\max\limits_{P\in\Delta}\mathbb{E}_P[H(\boldsymbol{\theta},\xi)]:=\mathbb{E}_P[R_{\xi}^{\rm{rob}}],
\end{equation}
where
\begin{equation}\label{eq:6}
    \Delta=\{P\in\mathbb{D}:D(P||P_0)\leq \eta\}
\end{equation}
is the ambiguity set, $\mathbb{D}$ is the set of all probability distributions, and $P_0$ represents an empirical distribution, which is typically assumed to be uniform. $D(P||P_0)$ represents the divergence between distribution $P$ and the empirical distribution $P_0$, $\eta$ is an important hyperparameter that determines the size of the set, i.e., the range of values that the distribution $P$ can assume. A larger $\eta$ often results in training failure. This is due to the fact that with a larger $\eta$, the permissible range of the distribution $P$ becomes too broad, leading the model to become excessively conservative and ultimately resulting in training failure. Conversely, a smaller $\eta$ may hinder the ability to explore the unknown distribution.

In this paper, we consider the $\chi^2$-divergence function, defined as
\begin{equation}
    D(P||P_0)=\int_{\Xi}\frac{[p(\xi)-p_0(\xi)]^2}{p_0(\xi)}d\xi. \nonumber
\end{equation}

In the context of image classification task, $P_0$ and $P$ are both discrete, so
\begin{equation}\label{eq:7}
    D(P||P_0)=\sum_{\xi\in\Xi}\frac{[p(\xi)-p_0(\xi)]^2}{p_0(\xi)}.
\end{equation}

It can be observed that~\cref{eq:5} presents a standard DRO problem, constrained by $\chi^2$-divergence. To obtain a closed-form optimal solution for the inner maximization problem, it is necessary that the objective function $H(\boldsymbol{\theta},\xi)$ satisfy certain assumptions.

\begin{assumption}\label{assumption:1}
    The variance of $H(\boldsymbol{\theta },\xi )$ under the empirical distribution $P_0$ is finite.
\end{assumption}

The variance of $H(\boldsymbol{\theta },\xi )$ is defined as
\begin{equation}
    \mathrm{Var}_{P_0}(H(\boldsymbol{\theta },\xi ))=\mathbb{E} _{P_0}[H^2(\boldsymbol{\theta },\xi )]-(\mathbb{E} _{P_0}[H(\boldsymbol{\theta },\xi )])^2. \nonumber
\end{equation}

\cref{assumption:1} stipulates that the first and second moments of $H(\boldsymbol{\theta },\xi )$ are finite under $P_0$. In light of the fact that $P_0$ is a discrete distribution and $H(\boldsymbol{\theta },\xi )=R_{\xi}^{\mathrm{rob}}$ is bounded, it can be concluded that $H(\boldsymbol{\theta },\xi )$ meets the assumption.

The following~\cref{theorem:1} presents the explicit formulation of the derived closed-form optimal solution.
\begin{theorem}\label{theorem:1}
    If~\cref{assumption:1} holds true, the closed-form optimal solution to the inner maximization problem within~\cref{eq:5} is
    \begin{equation}\label{eq:8}
        \begin{split}
            p^*(\xi )&=p_0(\xi )\\
        &+p_0(\xi )\sqrt{\frac{\eta}{\mathrm{Var}_{P_0}[H(\boldsymbol{\theta },\xi )]}}\left\{ H(\boldsymbol{\theta },\xi )-\mathbb{E} _{P_0}[H(\boldsymbol{\theta },\xi )] \right\}.
        \end{split}
    \end{equation}
\end{theorem}

The proof of~\cref{theorem:1} is detailed in Appendix~\ref{append:1}. It is noteworthy that, unlike~\cite{2020arXiv200610138Q}, we have simultaneously derived the closed-form optimal solutions for both the distribution $P$ and the Lagrange multiplier $\alpha$ (which are introduced during the application of the Lagrange multiplier method), as shown in~\cref{eq:31} and~\cref{eq:33}. This improves the modeling capabilities of the DRO model.

Substituting the closed-form optimal solution $p^*(\xi )$ into~\cref{eq:5} converts the inner maximization problem into a deterministic one:
\begin{equation}\label{eq:9}
    \max_{P\in \Delta} \mathbb{E} _P[H(\boldsymbol{\theta },\xi )]=\mathbb{E} _{P_0}[H(\boldsymbol{\theta },\xi )]+\sqrt{\eta \mathrm{Var}_{P_0}(H(\boldsymbol{\theta },\xi ))}.
\end{equation}

Similarly, the inner maximization problem of CODAT can be transformed into a deterministic problem as follows, which we refer to as the deterministic equivalent objective function:
\begin{equation}\label{eq:10}
    \max_{P\in \mathbb{P}} \mathbb{E} _P[R_{\xi}^{\mathrm{rob}}]=\mathbb{E} _{P_0}[R_{\xi}^{\mathrm{rob}}]+\sqrt{\eta \mathrm{Var}_{P_0}(R_{\xi}^{\mathrm{rob}})}.
\end{equation}

In summary, the CODAT framework, as presented by~\cref{eq:4}, is equivalent to
\begin{equation}\label{eq:11}
    \min_{\theta} \mathbb{E} _{P_0}[R_{\xi}^{\mathrm{rob}}]+\sqrt{\eta \mathrm{Var}_{P_0}(R_{\xi}^{\mathrm{rob}})},
\end{equation}
and the closed-form optimal solution of the inner maximization problem is
\begin{equation}\label{eq:12}
    \begin{split}
        p^*(\xi )&=p_0(\xi )\\
        &+p_0(\xi )\sqrt{\frac{\eta}{\mathrm{Var}_{P_0}[R_{\xi}^{\mathrm{rob}}]}}[R_{\xi}^{\mathrm{rob}}-\mathbb{E} _{P_0}[R_{\xi}^{\mathrm{rob}}]].
    \end{split}
\end{equation}

From~\cref{eq:12}, it can be observed that if a class has a high risk $R_{\xi}^{\mathrm{rob}}$, CODAT assigns a large weight to it. Consequently, CODAT enables the model to focus more on the classes that are more challenging to learn, thereby enhancing the model's robust fairness. The details of CODAT are presented in~\cref{pseudocode}.

\begin{algorithm}[H]
    \caption{Class Optimal Distribution Adversarial Training}
    \label{pseudocode}
    \begin{algorithmic}[1]
    \renewcommand{\algorithmicrequire}{\textbf{Input:}}
    \renewcommand{\algorithmicensure}{\textbf{Output:}}
    \Require training set $\mathcal{S} =\{(\boldsymbol{x}_i,y_i)\}_{i=1}^{n}$, learning rate $\gamma$, DNN classifier $f_{\boldsymbol{\theta }}$ parameterized by $\boldsymbol{\theta}$, number of epochs $T$, empirical distribution $P_0$, the threshold of the ambiguity set $\eta$, number of classes $K$, batch size $m$ and loss function $l$.
    \Ensure a robust fair model $\overline{f_{\boldsymbol{\theta }}}$.
    \State Randomly initialize $\boldsymbol{\theta}$, and set $P_0$ to be a uniform distribution.
    \For{$1\leq t \leq T$}
        \State Sample a minibatch $(\mathbf{x},\mathbf{y}):=\{(\boldsymbol{x}_i,y_i)\}_{i=1}^{m}$ from $\mathcal{S}$
        % \Statex /*\textit{generate adversarial examples}*/
        \State $\mathbf{x}'\leftarrow \arg\max\limits_{\mathbf{x}'\in\mathcal{B}(\mathbf{x},\epsilon)}l(f_{\boldsymbol{\theta}}(\mathbf{x}'),\mathbf{y})$

        \State $l_{\rm{adv}}=l(f_{\boldsymbol{\theta}}(\mathbf{x}'),\mathbf{y})$
        \Statex /*\textit{The Class\_avg\_loss function calculates the average loss for each class}*/
        \For{$1\leq\xi\leq K$}
            \State $R_{\xi}^{\mathrm{rob}}=\mathrm{Class}\_\mathrm{avg}\_\mathrm{loss(}l_{\mathrm{adv}})$
        
        \EndFor
        % \Statex /*\textit{compute the equivalent loss as in~\cref{eq:11}}*/
        \State $l_{\mathrm{equal}}=\mathbb{E}_{P_0}[R_{\xi}^{\mathrm{rob}}]+\sqrt{\eta \mathrm{Var}_{P_0} (R_{\xi}^{\mathrm{rob}})}$
        % \Statex /*\textit{update model parameters}*/
        \State $\boldsymbol{\theta }\gets \boldsymbol{\theta }-\gamma \nabla_{\boldsymbol{\theta }}l_{\mathrm{equal}}$ 
        
    \EndFor
    \State return $\overline{f_{\boldsymbol{\theta }}}$
    \end{algorithmic}
\end{algorithm}

\subsection{The unnecessity of regularizers for degradation prevention in CODAT}\label{sec:4-3}
For the inner maximization problem within~\cref{eq:4}, it is evident that
\begin{equation}
    \max\limits_{P\in\mathbb{P}}\mathbb{E}_P[R_{\xi}^{\rm{rob}}]\leq \max\limits_{k\in[K]}R_k^{\rm{rob}}. \nonumber
\end{equation}

This indicates that the optimal solution to the maximization problem on the left-hand side may degenerate into a Dirac distribution in the absence of any precautions. As previously discussed in~\cref{rmk:1}, in such circumstances, the class adversarial distribution is unable to capture a broader range of weights, leading CODAT to revert to the formulation presented in~\cref{eq:3}. Furthermore, the Dirac distribution may result in instability of the model optimization routine and may often compromise the model's generalization capabilities~\cite{wang2021adversarial}.

To address this issue, many existing methods~\cite{2020arXiv200610138Q} incorporate a regularization term, such as the KL-divergence, to prevent the degeneration of the optimal solution into a Dirac distribution, as follows:
\begin{equation}\label{eq:13}
    \min\limits_{\boldsymbol{\theta}}\max\limits_{P\in\mathbb{P}}\mathbb{E}_P[H(\boldsymbol{\theta},\xi)-\lambda_P{\rm{KL}}(P||P_0)],
\end{equation}
where ${\rm{KL}}(\cdot)$ denotes the KL-divergence. The regularization term ${\rm{KL}}(P||P_0)$ encourages $P$ to approach $P_0$ during optimization, thereby preventing $P$ from degenerating into a Dirac distribution. However, \cref{eq:13} represents a relaxed form of the DRO problem, and this relaxation may compromise its modeling capability. Additionally, in the actual optimization process, the parameter $\lambda_P$ in~\cref{eq:13} is often treated as a constant. However, in the standard DRO problem, $\lambda_P$ itself is subject to optimization. Therefore, the solution obtained from~\cref{eq:13} may not be optimal.

As described in~\cref{sec:4-3}, we consider~\cref{eq:4} as a distributionally robust optimization problem constrained by $\chi^2$--divergence. Assuming a dataset with $K$ classes, let $P$ be a Dirac distribution, and without loss of generality, let $P=[1,0,0,\cdots,0]$. Let $P_0$ be the uniform distribution over the $K$ classes, such that each class has a probability of $\frac{1}{K}$ under $P_0$, i.e., $P_0=[\frac{1}{K},\frac{1}{K},\cdots,\frac{1}{K}]$. Then, the $\chi^2$-divergence between $P$ and $P_0$ is
\begin{equation}\label{eq:14}
    \begin{split}
        D(P||P_0) &= \sum\frac{[p(\xi)-p_0(\xi)]^2}{p_0(\xi)} \\
        &=\frac{[1-\frac{1}{K}]^2}{\frac{1}{K}}+(K-1)\cdot \frac{[0-\frac{1}{K}]^2}{\frac{1}{K}} \\
        &=K-1.
    \end{split}
\end{equation}

In accordance with~\cref{eq:14}, when $K=10$, the $\chi^2$-divergence between a Dirac distribution and a uniform distribution is 9. Similarly, with $K=100$, this divergence increases to 99. However, our experiments, as detailed in Section~\ref{sec:5-5}, indicate that the typical value selected for the constraint $\eta$ (defined in~\cref{eq:6}) on the $\chi^2$-divergence is less than 1, a value that is significantly lower than the $\chi^2$-divergence between a Dirac distribution and a uniform distribution. Consequently, it is unlikely that $P$ is a Dirac distribution. Therefore, our CODAT framework does not require the incorporation of a regularization term to prevent degeneration. 

\subsection{Fairness elasticity coefficient}\label{sec:4-4}
A number of studies~\cite{benz2021robustness,li2023wat} have demonstrated a trade-off between model robustness and fairness, indicating that improvements in robust fairness often result in a reduction in the robustness. To characterize the interrelationship between alterations in a model's fairness and the ensuing variations in its robustness, we propose the Fairness Elasticity Coefficient (FEC), a novel metric drawing inspiration from the economic concept of the elasticity coefficient, which is employed to measure the sensitivity of one variable to changes in another. Specifically, the objective of FEC is to quantify the relationship between the change in the robust accuracy of the worst class and the corresponding variation in the average robust accuracy. The definition of FEC is detailed below.
\begin{definition}
    Let $\Delta A_{\rm{wc}}=\frac{A_{\rm{wc}}-A_{\rm{wc}}^{\rm{bl}}}{A_{\rm{wc}}^{\rm{bl}}}$ denote the rate at which the worst-class accuracy improves relative to a baseline method for a given approach. Let $\Delta \overline{A}=\frac{\overline{A^{\rm{bl}}}-\overline{A}}{\overline{A^{\rm{bl}}}}$ represent the rate at which the average accuracy decreases relative to the baseline method. The Fairness Elasticity Coefficient (FEC) is then defined as
    \begin{equation}\label{eq:15}
        {\rm{FEC}}=\frac{e^{\Delta A_{\rm{wc}}}}{e^{\Delta \overline{A}}}.
    \end{equation}
\end{definition}

It can be observed that when ${\rm{FEC}}>1$, it indicates that the average robust accuracy of the model declines at a relatively slower rate compared to the improvement in the robust accuracy of the worst class. Conversely, if ${\rm{FEC}}<1$, the opposite is true. A detailed analysis of FEC is provided in Appendix~\ref{append:2}. Therefore, the FEC offers a comprehensive measure that captures the ability of the algorithm to increase robust fairness at the cost of decreasing average robustness. Obviously, a higher elasticity coefficient signifies better performance of the algorithm.

\section{Experiments}\label{sec:5}
In this section, we present the results of extensive experiments conducted across various datasets and models to evaluate the effectiveness of our proposed method in improving the robust fairness. We begin by detailing the experimental setups employed in our study. Subsequently, we present comparative analyses with state-of-the-art methods. Finally, we perform ablation studies to gain deeper insights into our framework.

\subsection{Experimental setups}\label{sec:5-1}
\begin{table*}[ht]
\caption{Comparative Performance of all methods on CIFAR-10 using ResNet-18. We evaluate the average accuracy (\%) and worst-class accuracy (\%) on natural data and under PGD-100, CW-30, and AA attacks. The best results are highlighted in \textbf{bold}.}
\label{tab:1}
\resizebox{\linewidth}{!}{
\begin{tabular}{cccccccccccccccc}
\hline
CIFAR-10 & \multicolumn{3}{c}{Natural}                     &           & \multicolumn{3}{c}{PGD-100}                     &           & \multicolumn{3}{c}{CW-30}                       &           & \multicolumn{3}{c}{AutoAttack}                  \\ \cline{2-4} \cline{6-8} \cline{10-12} \cline{14-16} 
Method   & Avg.           & Wst.           & FEC           &           & Avg.           & Wst.           & FEC           &           & Avg.           & Wst.           & FEC           &           & Avg.           & Wst.           & FEC           \\ \hline
AT       & 83.81 & 65.40 & 1.00 &           & 49.57          & 22.00          & 1.00          &           & 49.87 & 22.40          & 1.00          &           & 47.01          & 18.80          & 1.00          \\
TRADES   & 82.38          & 68.30          & 1.03          &           & \textbf{52.55} & 28.50          & 1.43          &           & \textbf{50.93} & 25.90          & 1.19          &           & \textbf{49.52} & 24.50          & 1.43          \\
FRL-RWRM & 83.47 & 71.10 & 1.09 &           & 48.70          & 30.60          & 1.45          &           & 48.28          & 30.90          & 1.42          &           & 46.09          & 26.40          & 1.47          \\
BAT      & \textbf{86.46} & \textbf{75.20} & \textbf{1.20} &  & 48.30          & 26.10 & 1.17 &           & 47.40          & 24.00          & 1.02          &           & 44.90          & 21.90          & 1.13          \\
CFOL     & 81.37          & 64.70          & 0.96          &           & 47.46          &32.30 & 1.53          &           & 46.03          & 25.60 & 1.07          &           & 43.01          & 22.10 & 1.09          \\
WAT      & 80.05 & 65.20          & 0.95          &           & 49.69 & 36.30          & 1.92          &           & 48.25 & 33.10          & 1.56          &           & 46.80 & 30.90          & 1.89          \\
CODAT    & 80.62          & 66.40          & 0.98          &           & 50.56          & \textbf{37.30} & \textbf{2.05} &  & 47.87          & \textbf{34.10} & \textbf{1.62} &  & 46.57          & \textbf{32.30} & \textbf{2.03} \\ \hline
\end{tabular}}
\end{table*}

\begin{table*}[ht]
\caption{Comparative Performance of all methods on CIFAR-100 using ResNet-18. We evaluate the average accuracy (\%) and worst-class accuracy (\%) on natural data and under PGD-100, CW-30, and AA attacks. The best results are highlighted in \textbf{bold}.}
\label{tab:2}
\resizebox{\linewidth}{!}{
\begin{tabular}{cccccccccccccccc}
\hline
CIFAR-100 & \multicolumn{3}{c}{Natural}                     &           & \multicolumn{3}{c}{PGD-100}                    &           & \multicolumn{3}{c}{CW-30}                      &           & \multicolumn{3}{c}{AutoAttack}                 \\ \cline{2-4} \cline{6-8} \cline{10-12} \cline{14-16} 
Method    & Avg.           & Wst.           & FEC           &           & Avg.           & Wst.          & FEC           &           & Avg.           & Wst.          & FEC           &           & Avg.           & Wst.          & FEC           \\ \hline
AT        &57.52 & 20.00 & 1.00 &           & 24.36          & 2.00          & 1.00          &           & 24.04 & 2.00          & 1.00          &           & 22.19          & 1.00          & 1.00          \\
TRADES    & 55.24          & 17.00          & 0.83          &           & 27.89 & 2.00          & 1.16          &           & 25.05 & 1.00          & 0.63          &           & 24.07 & 1.00          & 1.09          \\
FRL-RWRM  & 53.40 & \textbf{24.00} & 1.14 &           & 22.68          & 2.00          & 0.93          &           & 21.46          & 2.00          & 0.90          &           & 19.99          & 1.00          & 0.91          \\
BAT       & \textbf{61.74} & 22.00 & \textbf{1.19} &  & 28.71          & \textbf{4.00} & \textbf{3.25} &           & 24.32          & 1.00          & 0.61          &           & 22.88          & 1.00          & 1.03          \\
CFOL      & 53.51          & 20.00          & 0.93          &           & 24.21          & \textbf{4.00} & 2.70          &           & 22.92          & \textbf{3.00} & 1.57          &           & 20.74          & \textbf{2.00} & 2.55          \\
WAT       & 53.00 & 22.00          & 1.02          &           & \textbf{29.00} & 3.00          & 1.99          &           & \textbf{27.00} & 2.00          & 1.13          &           & \textbf{26.00} & 1.00          & 1.19          \\
CODAT     & 55.90          & 20.00          & 0.97          &           & 27.60          & \textbf{4.00} & 3.10 &  & 23.83          & \textbf{3.00} & \textbf{1.63} &  & 22.92          & \textbf{2.00} & \textbf{2.81} \\ \hline
\end{tabular}}
\end{table*}

\begin{table*}[ht]
\caption{Comparative Performance of all methods on SVHN using ResNet-18. We evaluate the average accuracy (\%) and worst-class accuracy (\%) on natural data and under PGD-100, CW-30, and AA attacks. The best results are highlighted in \textbf{bold}.}
\label{tab:3}
\resizebox{\linewidth}{!}{
\begin{tabular}{cccccccccccccccc}
\hline
SVHN     & \multicolumn{3}{c}{Natural}                     &           & \multicolumn{3}{c}{PGD-100}                     &           & \multicolumn{3}{c}{CW-30}                       &           & \multicolumn{3}{c}{AutoAttack}                  \\ \cline{2-4} \cline{6-8} \cline{10-12} \cline{14-16} 
Method   & Avg.           & Wst.           & FEC           &           & Avg.           & Wst.           & FEC           &           & Avg.           & Wst.           & FEC           &           & Avg.           & Wst.           & FEC           \\ \hline
AT       & 92.74 & 87.20 & 1.00 &           & 53.18          & 35.78          & 1.00          &           & 51.61 & 36.81          & 1.00          &           & 47.27          & 31.93          & 1.00          \\
TRADES   & 88.30          & 76.13          & 0.84          &           & \textbf{56.43} & 35.54          & 1.06          &           & \textbf{53.15} & 32.17          & 0.91          &           & \textbf{50.52} & 29.46          & 0.99          \\
FRL-RWRM & 93.65 & \textbf{89.10} & \textbf{1.03} &           & 54.48          & 39.16          & 1.13          &           & 51.86          & 36.51          & 1.00          &           & 47.02          & 33.19          & 1.03          \\
BAT      & 90.19 & 79.77 & 0.89 &  & 41.55          & 29.25 & 0.67 &           & 31.78          & 19.70          & 0.43          &           & 24.59          & 15.78          & 0.37          \\
CFOL     & 89.85          & 82.27          & 0.92          &           & 44.10          & 32.76 & 0.77          &           & 40.63          & 31.05 & 0.69          &           & 37.04          & 27.55 & 0.70          \\
WAT      & \textbf{93.73} & 88.17          & 1.02          &           & 54.73 & 41.20          & 1.20          &           & 50.53 & 37.17          & 0.99          &           & 47.02 & 33.25          & 1.04          \\
CODAT    & 91.69          & 87.86          & 1.00          &           & 54.73          & \textbf{46.91} & \textbf{1.41} &  & 48.98          & \textbf{40.53} & \textbf{1.05} &  & 46.27          & \textbf{38.55} & \textbf{1.20} \\ \hline
\end{tabular}}
\end{table*}

\begin{table*}[ht]
\caption{Comparative Performance of all methods on STL-10 using ResNet-18. We evaluate the average accuracy (\%) and worst-class accuracy (\%) on natural data and under PGD-100, CW-30, and AA attacks. The best results are highlighted in \textbf{bold}.}
\label{tab:4}
\resizebox{\linewidth}{!}{
\begin{tabular}{cccccccccccccccc}
\hline
STL-10   & \multicolumn{3}{c}{Natural}                     &           & \multicolumn{3}{c}{PGD-100}                     &           & \multicolumn{3}{c}{CW-30}                       &           & \multicolumn{3}{c}{AutoAttack}                  \\ \cline{2-4} \cline{6-8} \cline{10-12} \cline{14-16} 
Method   & Avg.           & Wst.           & FEC           &           & Avg.           & Wst.           & FEC           &           & Avg.           & Wst.           & FEC           &           & Avg.           & Wst.           & FEC           \\ \hline
AT       & 64.31          & 38.00          & 1.00          &           & 36.88          & 12.13          & 1.00          &           & \textbf{35.45} & 7.63           & 1.00          &           & 33.88          & 5.75           & 1.00          \\
TRADES   & 63.95          & 40.00          & 1.05          &           & \textbf{38.06} & 12.25          & 1.04          &           & 35.15 & 9.63           & 1.29          &           & \textbf{34.68} & 9.00           & 1.80          \\
FRL-RWRM & \textbf{65.58} & \textbf{46.25} & \textbf{1.27} &           & 33.80          & 15.63          & 1.23          &           & 31.94          & 12.50          & 1.71          &           & 31.64          & 10.50          & 2.14          \\
BAT      & 60.40 & 40.13 & 1.00 &  & 27.43          & 7.50  & 0.53 &           & 22.66          & 4.25           & 0.45          &           & 21.98          & 3.88           & 0.51          \\
CFOL     & 51.58          & 34.50          & 0.75          &           & 26.45          & 14.75 & 0.94          &           & 24.19          & 8.38  & 0.80          &           & 22.99          & 6.13  & 0.77          \\
WAT      & 59.31 & 43.63          & 1.07          &           & 32.91 & 14.63          & 1.10          &           & 29.93 & 12.00          & 1.52          &           & 29.66 & 11.63          & 2.45          \\
CODAT    & 60.06          & 46.00          & 1.16          &           & 33.31          & \textbf{17.50} & \textbf{1.41} &  & 30.01          & \textbf{13.75} & \textbf{1.91} &  & 29.54          & \textbf{13.25} & \textbf{3.24} \\ \hline
\end{tabular}}
\end{table*}

\begin{table*}[ht]
\caption{Comparative Performance of all methods on CIFAR-10 using WideResNet-34-10. We evaluate the average accuracy (\%) and worst-class accuracy (\%) on natural data and under PGD-100, CW-30, and AA attacks. The best results are highlighted in \textbf{bold}.}
\label{tab:5}
\resizebox{\linewidth}{!}{
\begin{tabular}{cccccccccccccccc}
\hline
CIFAR-10 & \multicolumn{3}{c}{Natural}                     &           & \multicolumn{3}{c}{PGD-100}                     &           & \multicolumn{3}{c}{CW-30}                       &           & \multicolumn{3}{c}{AutoAttack}                  \\ \cline{2-4} \cline{6-8} \cline{10-12} \cline{14-16} 
Method   & Avg.           & Wst.           & FEC           &           & Avg.           & Wst.           & FEC           &           & Avg.           & Wst.           & FEC           &           & Avg.           & Wst.           & FEC           \\ \hline
AT       & \textbf{86.62} & \textbf{75.20} & \textbf{1.00} &           & 48.84          & 23.30          & 1.00          &           & 50.19 & 25.40          & 1.00          &           & 47.51          & 22.30          & 1.00          \\
TRADES   & 84.52          & 71.50          & 0.93          &           & \textbf{54.61} & 31.50          & 1.60          &           & \textbf{53.97} & 30.70          & 1.33          &           & \textbf{52.19} & 27.70          & 1.41          \\
FRL-RWRM & 85.04 & 70.70 & 0.92 &           & 50.68          & 28.90          & 1.32          &           & 51.60          & 29.60          & 1.21          &           & 49.40          & 27.40          & 1.31          \\
BAT      & 86.08 & 74.20 & 0.98 &  & 47.44          & 27.70 & 1.17 &           & 49.20          & 30.10          & 1.18          &           & 45.24          & 25.00          & 1.08          \\
CFOL     & 86.14          & 71.40          & 0.95          &           & 50.60          & 31.10 & 1.45          &           & 50.64          & 31.80 & 1.30          &           & 48.16          & 27.80 & 1.30          \\
WAT      & 83.66 & 71.90          & 0.92          &           & 52.42 & 33.60          & 1.67          &           & 51.67 & 32.10          & 1.34          &           & 50.10 & 30.30          & 1.51          \\
CODAT    & 83.89          & 73.10          & 0.94          &           & 53.97          & \textbf{39.00} & \textbf{2.18} &  & 51.55          & \textbf{34.90} & \textbf{1.49} &  & 50.16          & \textbf{33.40} & \textbf{1.74} \\ \hline
\end{tabular}}
\end{table*}

\begin{figure*}
    \centering
    \subfloat[AT and CODAT]{
        \includegraphics[scale=0.22]{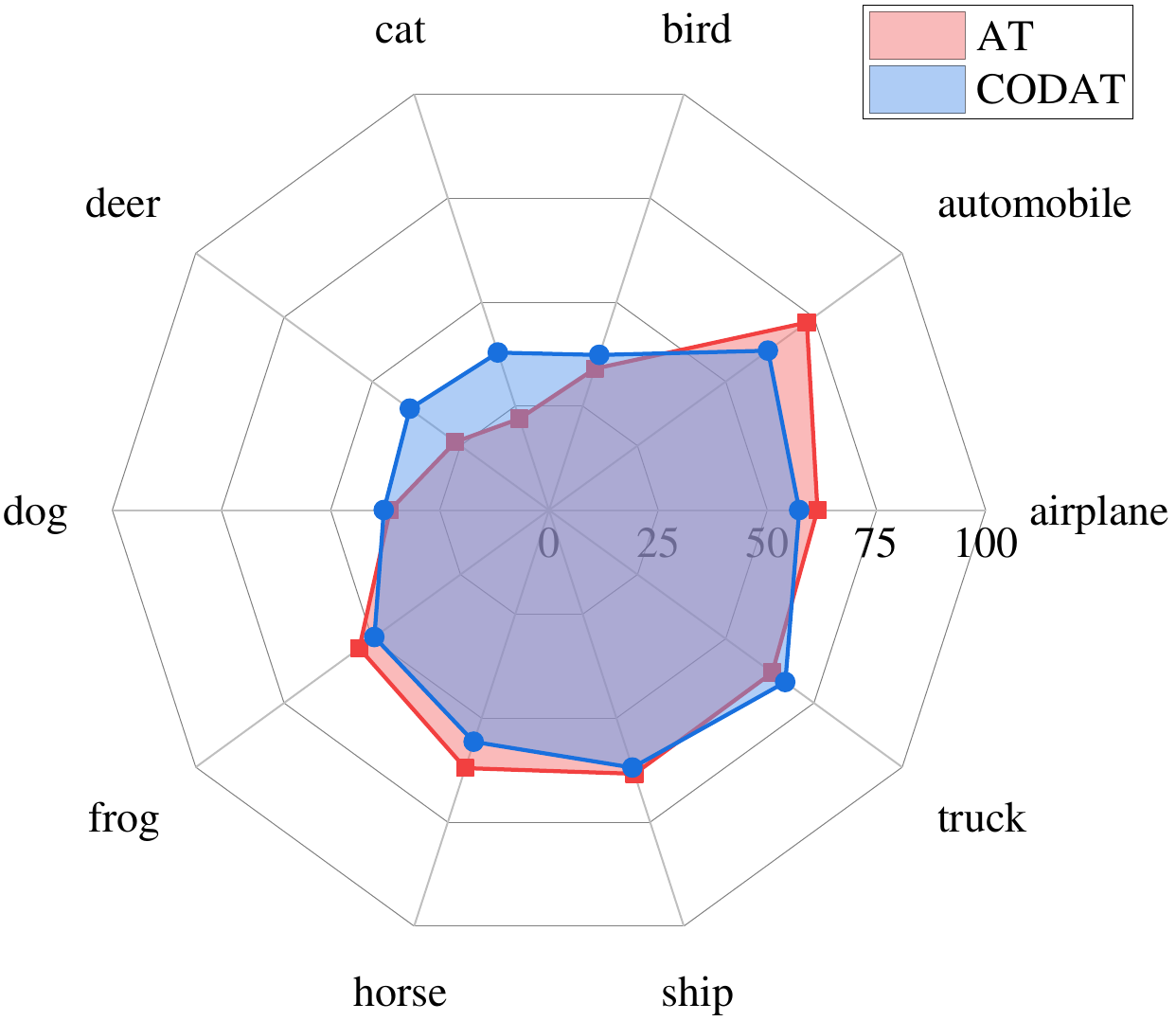}
        \label{fig:4-1}
    }
    \hfil
    \subfloat[TRADES and CODAT]{
        \includegraphics[scale=0.22]{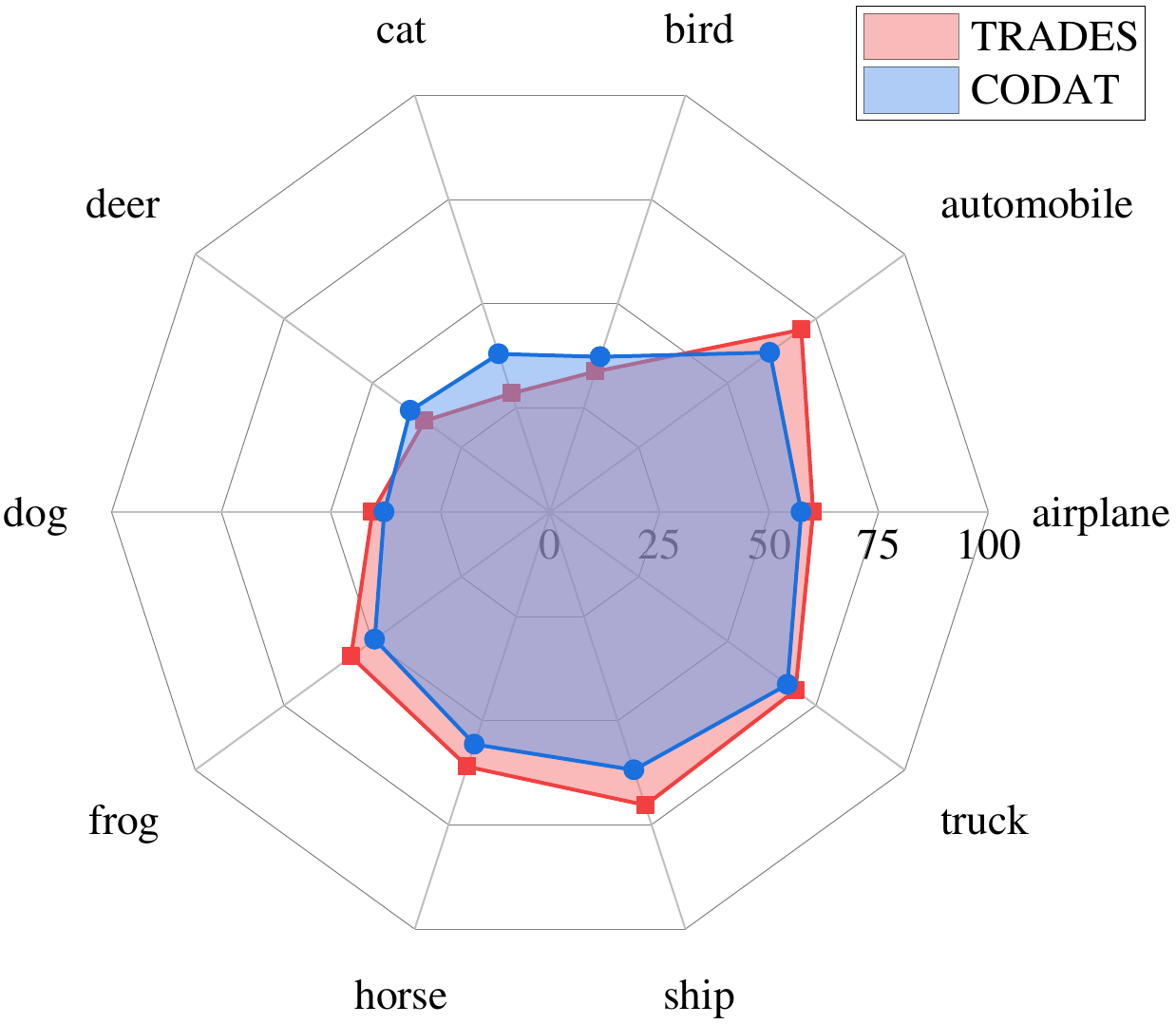}
        \label{fig:4-2}
    }
    \hfil
    \subfloat[FRL-RWRM and CODAT]{
        \includegraphics[scale=0.22]{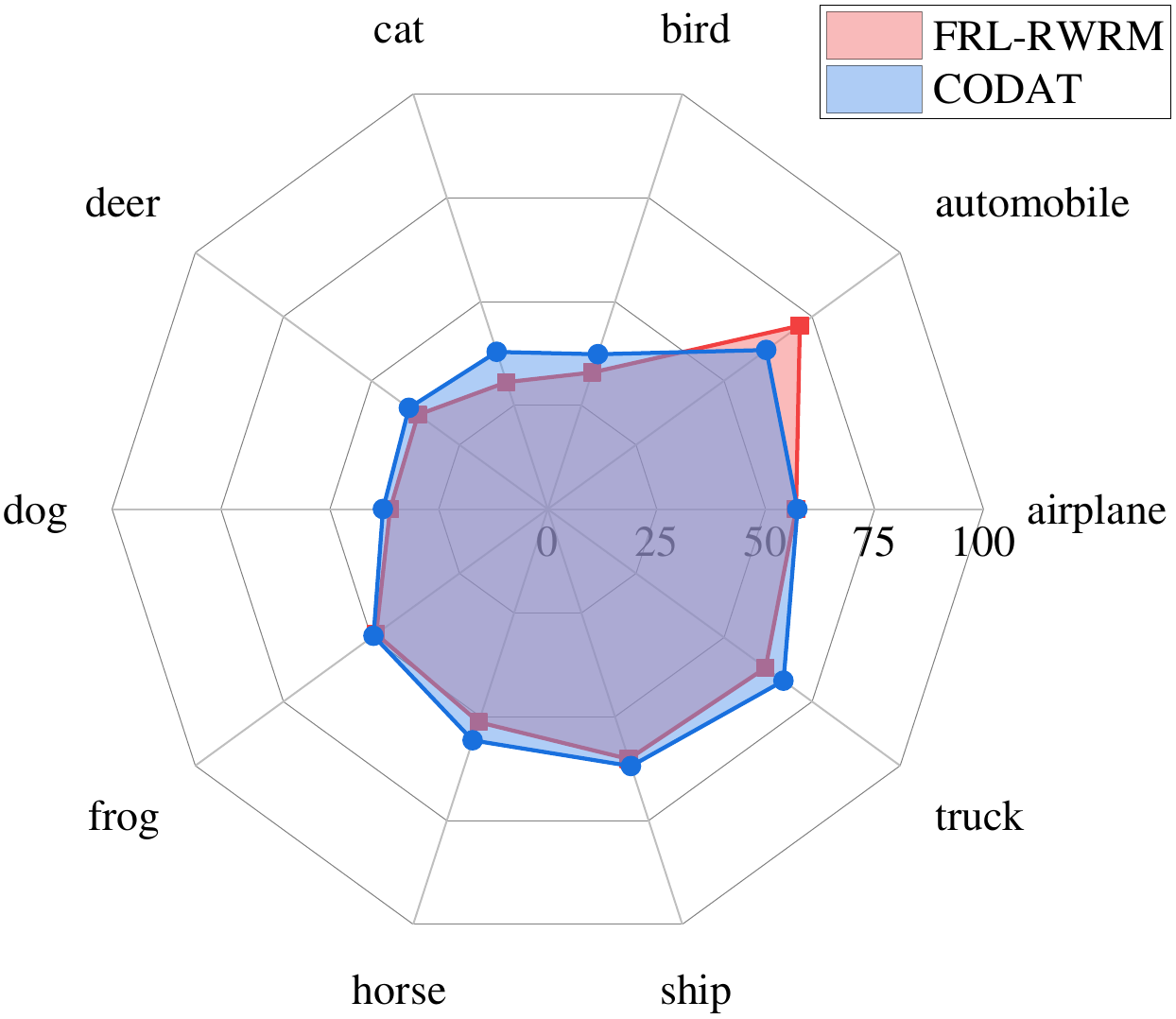}
        \label{fig:4-3}
    }

    \subfloat[BAT and CODAT]{
        \includegraphics[scale=0.22]{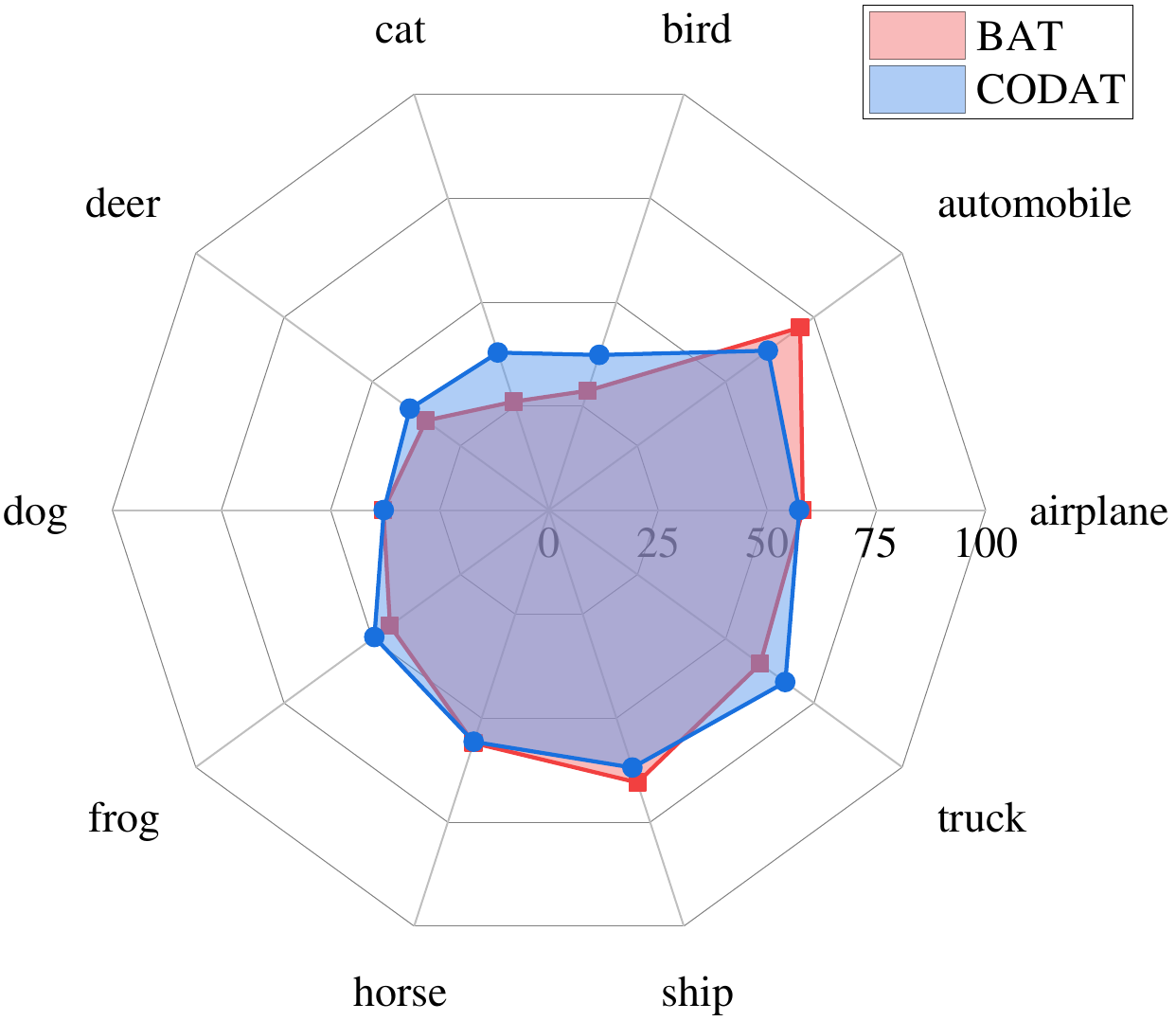}
        \label{fig:4-4}
    }
    \hfil
    \subfloat[CFOL and CODAT]{
        \includegraphics[scale=0.22]{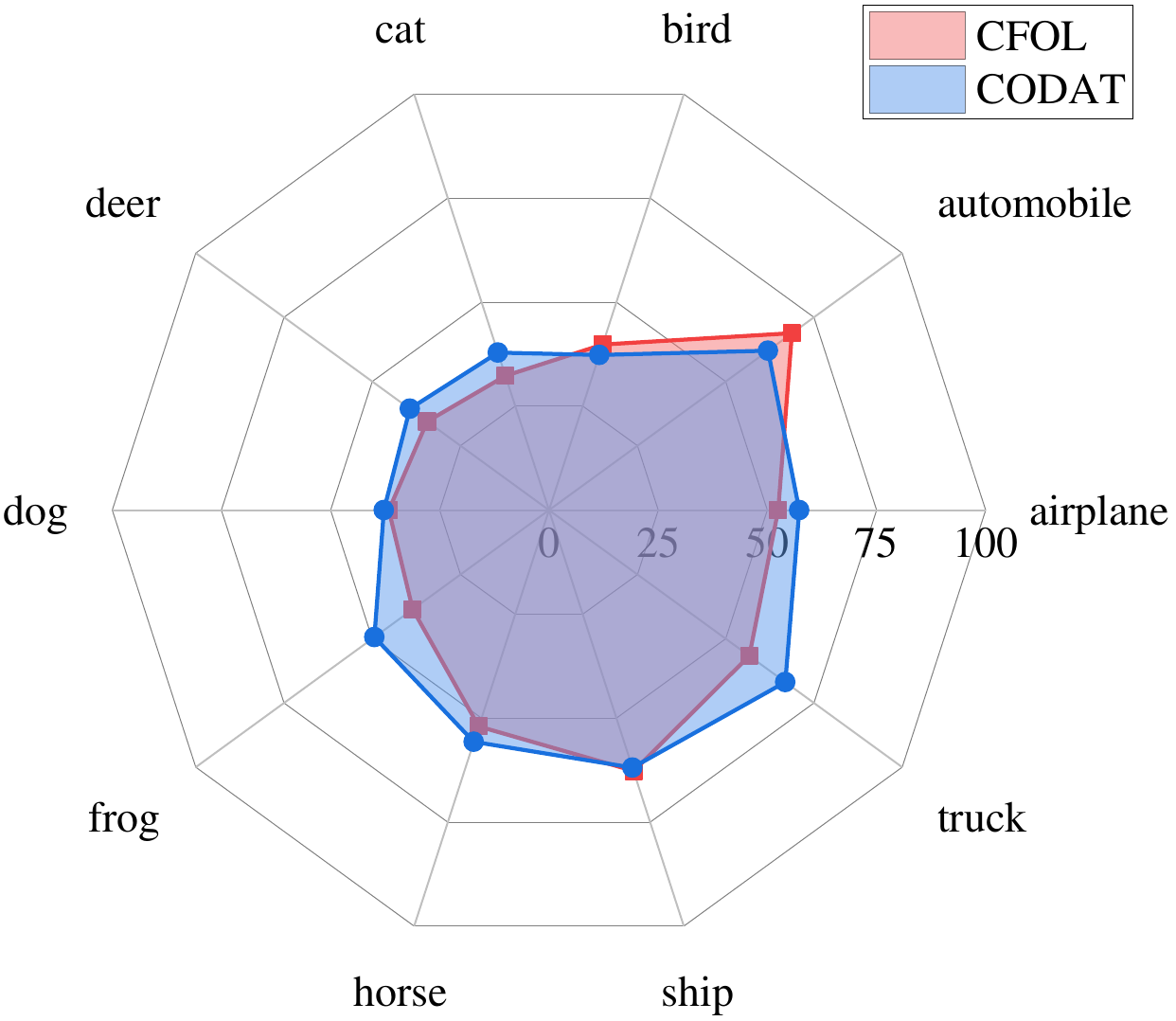}
        \label{5}
    }
    \hfil
    \subfloat[WAT and CODAT]{
        \includegraphics[scale=0.22]{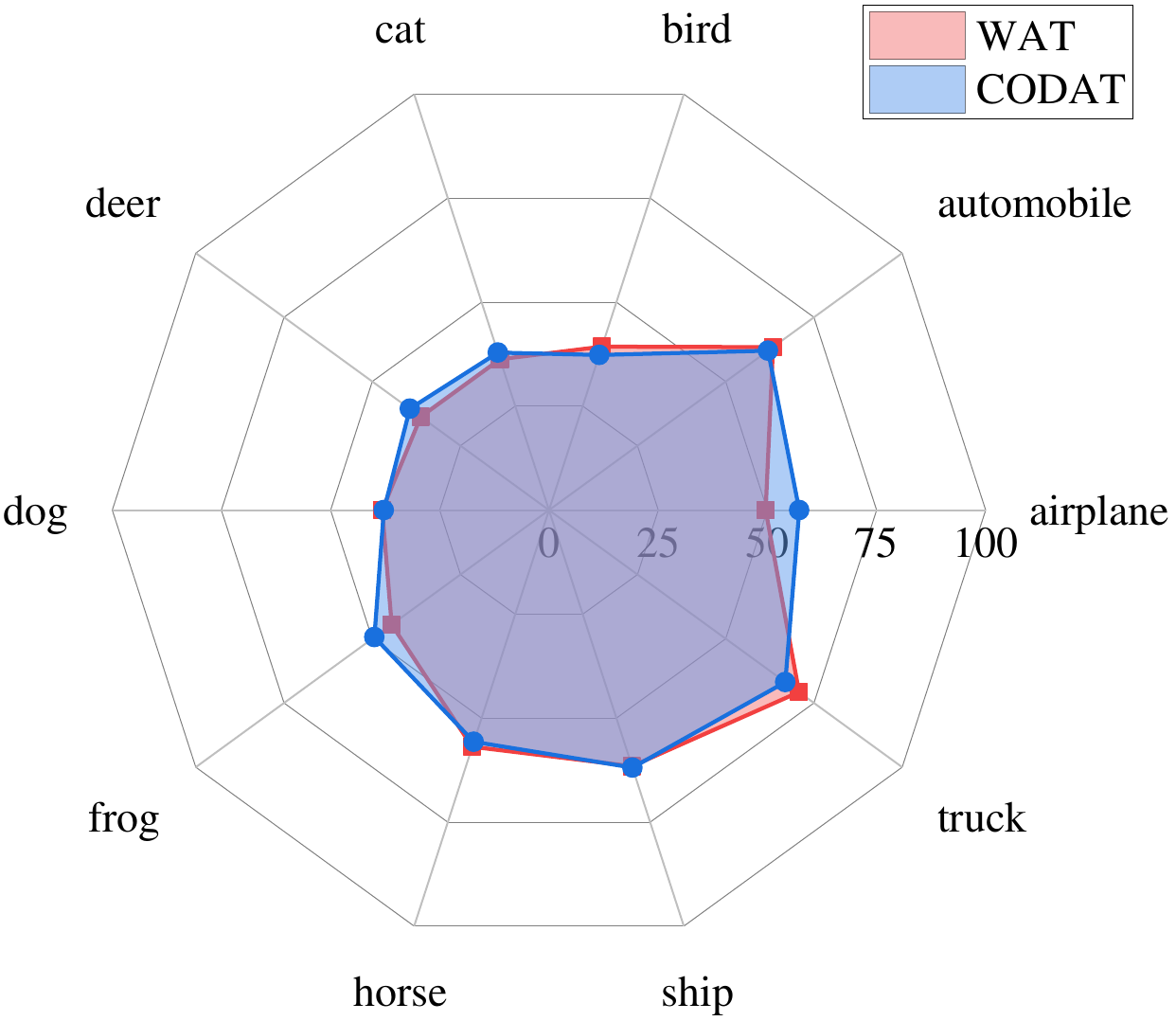}
        \label{fig:4-6}
    }
    \caption{Comparative analysis of class-wise robust accuracy for our method and baselines on CIFAR-10 using ResNet-18 under PGD-100 attack.}
    \label{fig:4}
\end{figure*}

All the experiments are conducted on an Nvidia 3090 GPU, with Ubuntu 18.04 LTS as the operating system and PyTorch 1.12 employed as the deep learning framework.

\textbf{Datasets and architectures.} We conduct experiments on four benchmark datasets: CIFAR-10~\cite{krizhevsky2009learning}, CIFAR-100~\cite{krizhevsky2009learning}, SVHN~\cite{netzer2011reading}, and STL-10~\cite{Coates_Ng_Lee_2011}. Further details regarding each dataset are provided in Appendix~\ref{append:3}. In our experiments, we employ the ResNet-18~\cite{He_Zhang_Ren_Sun_2016} and ResNet-34~\cite{Zagoruyko_Komodakis_2016} model architectures to evaluate the performance of all the algorithms.

\textbf{Baselines.} The objective of CODAT is to improve the robust fairness of models. To this end, we meticulously select three SOTA methods as baselines for comparison. These include FRL-RWRM~\cite{xu2021robust}, BAT~\cite{sun2023improving}, and the recently introduced WAT~\cite{li2023wat}. Furthermore, we compare our method with two traditional adversarial training methods: Adversarial Training (AT)~\cite{2017arXiv170606083M} and TRADES~\cite{Zhang_Yu_Jiao_Xing_Ghaoui_Jordan_2019}.

\textbf{Training settings.} Following the widely used adversarial training settings in~\cite{Zhang_Yu_Jiao_Xing_Ghaoui_Jordan_2019}, we train all the models using SGD optimizer with momentum 0.9, weight decay $2\times 10^{-4}$, and batch size 128 for 100 epochs. The initial learning rate for experiments, except for those on the SVHN dataset where it is set to 0.01, is uniformly set to 0.1 for other datasets. A piece-wise constant decay learning rate scheduler is used, where the learning rate is reduced by a factor of 0.1 at two predetermined epochs during training, namely the 75th and the 90th epoch. All experiments utilize the PGD attack with a random start to generate adversarial examples during the training phase. We set $\epsilon=8/255$, the step size to $2/255$, and the number of steps to 10. For data augmentation, we employ a combination of $32\times32$ random crop with 4-pixel padding and random horizontal flip. For TRADES, we set the hyper-parameter $\beta=6.0$. Regarding FRL-RWRM, we set $\tau_1=\tau_2=0.05$, $\alpha_1=\alpha_2=0.05$, and epochs to 100. As for CFOL, we set $\eta=2\times 10^{-6}$ and $\gamma=0.5$. All baseline methods are implemented using their official publicly available code.

\textbf{Evaluation settings.} We evaluate the performance of the algorithms in terms of average (Avg.) and worst-class (Wst.) accuracy under both natural and adversarial scenarios. In line with~\cite{li2023wat}, we use three strong adversarial attacks, including PGD~\cite{2017arXiv170606083M}, CW~\cite{Carlini_Wagner_2017}, and AutoAttack (AA)~\cite{Croce_Hein_2020}, to evaluate the robustness. In particular, AA is known for its reliability in robustness evaluation. We set $\epsilon=8/255$, step size to $1/255$, and number of steps for PGD to 100, CW to 30. Regarding AA, the standard version is employed. We also use the FEC as defined in~\cref{eq:15} to comprehensively evaluate the algorithm both in the average and worst-class accuracies.

\subsection{Fairness comparison with baselines}\label{sec:5-2}
We conduct a systematic evaluation of the robust fairness of each method. The performance of each method on CIFAR-10, CIFAR-100, SVHN, and STL-10 datasets with ResNet-18 is presented in Tables~\ref{tab:1} through~\ref{tab:4}, respectively. It can be observed that our method significantly outperforms the baselines across the four benchmark datasets. Following this observation, we proceed to present a comprehensive analysis of the experimental results for each dataset.

\textbf{CIFAR-10.} In comparison to the standard AT, all other methods have demonstrated an improvement in the worst-class robust accuracy, with CODAT shown the most significant enhancement, as depicted in~\cref{tab:1}. Specifically, CODAT outperforms other methods in terms of worst-class robust accuracy, demonstrating an improvement of at least 1\% under PGD-100 attack, a similar improvement under CW-30 attack, and achieving a minimum of 1.4\% increase under AA attack. Regarding average robust accuracy, while TRADES achieves the highest performance across all attacks, CODAT stands out with the largest FEC across all scenarios. This suggests that CODAT is capable of achieving the greatest improvement in worst-class robust accuracy with the least compromise on the average robust accuracy compared to the baseline methods.

\textbf{CIFAR-100.} As detailed in~\cref{tab:2}, CODAT substantially improves the worst-class robust accuracy. In fact, the worst-class robust accuracy of CODAT is superior to that of all the baselines under all three attacks. Although BAT achieves a comparable worst-class robust accuracy under PGD-100 attack, CODAT significantly outperforms it under CW-30 and AA attacks. CFOL achieves a level of performance that is comparable to that of CODAT in terms of worst-class robust accuracy. Nevertheless, the average robust accuracy of CODAT exceeds that of CFOL by 3.39\%, 0.91\%, and 2.18\% for each respective attack. This indicates that CODAT achieves comparable improvements in worst-class robust accuracy with a relatively smaller decrease in average robust accuracy compared to CFOL. Conversely, WAT exhibits a slight superiority in average robust accuracy, yet CODAT consistently maintains a higher worst-class robust accuracy. Moreover, CFOL exhibits the highest FEC across all attack scenarios. This suggests that, in comparison to WAT, CODAT improves worst-class robust accuracy at a more accelerated rate than the decline in average robust accuracy.

\textbf{SVHN and STL10.} Similar improvements are observed on SVHN and STL-10 datasets, as illustrated in Tables~\ref{tab:3} and~\ref{tab:4}. Compared to baselines, CODAT achieves a minimum improvement of 5.71\%, 3.36\%, and 5.30\% in worst-class robust accuracy under PGD-100, CW-30, and AA attacks on SVHN dataset, respectively, and a minimum improvement of 1.87\%, 1.25\%, and 1.62\% on STL10 dataset, respectively. Consistent with the observations on CIFAR-10 dataset, TRADES achieves the highest average robust accuracy on both SVHN and STL10 datasets, while CODAT exhibits the optimal FEC. This pattern reaffirms CODAT's capability to achieve the highest worst-class robust accuracy with the least compromise to average robust accuracy.

It can also be observed that there are variations in the performance of algorithms on natural data of different datasets. BAT exhibits superior performance in terms of both average natural accuracy and the worst-class natural accuracy on CIFAR-10 and CIFAR-100 datasets. In contrast, FRL-RWRM shows superior performance on SVHN and STL10. Nevertheless, CODAT consistently attains competitive performance. Moreover, the disparity in FEC between CODAT and BAT as well as FRL-RWRM on natural data is considerably smaller than that on adversarial examples. For instance, on CIFAR-100, the FEC gap between CODAT and BAT is -0.22 on natural data, yet this gap widens to 1.78 under AA attack scenario. This indicates that CODAT, in comparison to SOTA methods, can maintain a competitive performance on natural data with only a slight reduction, while significantly improving performance on adversarial data.

Furthermore, to further validate the effectiveness of CODAT, we evaluate the performance of all methods on CIFAR-10 dataset using a larger WideResNet-34-10 model, with the results detailed in~\cref{tab:5}. As illustrated in~\cref{tab:5}, with the employment of a larger model, CODAT can further increase the gap in worst-class robust accuracy relative to the baselines. Specifically, when using WideResNet-34-10 model, CODAT outperforms other methods in worst-class robust accuracy by a minimum improvement of 5.4\% under PGD-100 attack, 2.8\% under CW-30 attack, and 3.1\% under AA attack. This performance is significantly superior to those observed with ResNet-18 model. The notable improvement may be attributed to the fact that, with a larger model, CODAT is able to leverage more model capacity to focus on more challenging classes, which further demonstrates its superiority.

\subsection{Comparison of class-wise robust accuracy}\label{sec:5-3}
To conduct a fine-grained study on the effectiveness of CODAT in enhancing robust fairness, we perform a horizontal comparative analysis of class-wise robust accuracy with baseline methods on CIFAR-10 using ResNet-18 under PGD-100 attack in~\cref{fig:4}. As shown in~\cref{fig:4-1} and~\cref{fig:4-2}, although our method exhibits a slight decrease in robust accuracy for some well-performing classes compared to traditional adversarial training methods such as AT and TRADES, it achieves significant improvements in the most vulnerable classes, specifically cat and deer, thereby enhancing robust fairness. The remaining figures in~\cref{fig:4} demonstrate that, with the exception of the class automobile, our method outperforms SOTA methods in almost all classes. This provides substantial evidence of the effectiveness of CODAT in improving robust fairness. 

Moreover, \cref{fig:4} illustrates that the robust accuracy curves of CODAT across different classes approximate a circular shape, which indicating a balanced performance across all classes. This observation is further supported by the class-wise robust accuracy variance in~\cref{fig:5}. \cref{fig:5} presents the variance in class-wise robust accuracy for all methods under PGD-100, CW-30, and AA attacks, among which CODAT exhibits the lowest variance among all attacks. This finding further substantiates the effectiveness of CODAT.

\begin{figure}
    \centering
    \includegraphics[width=0.9\linewidth]{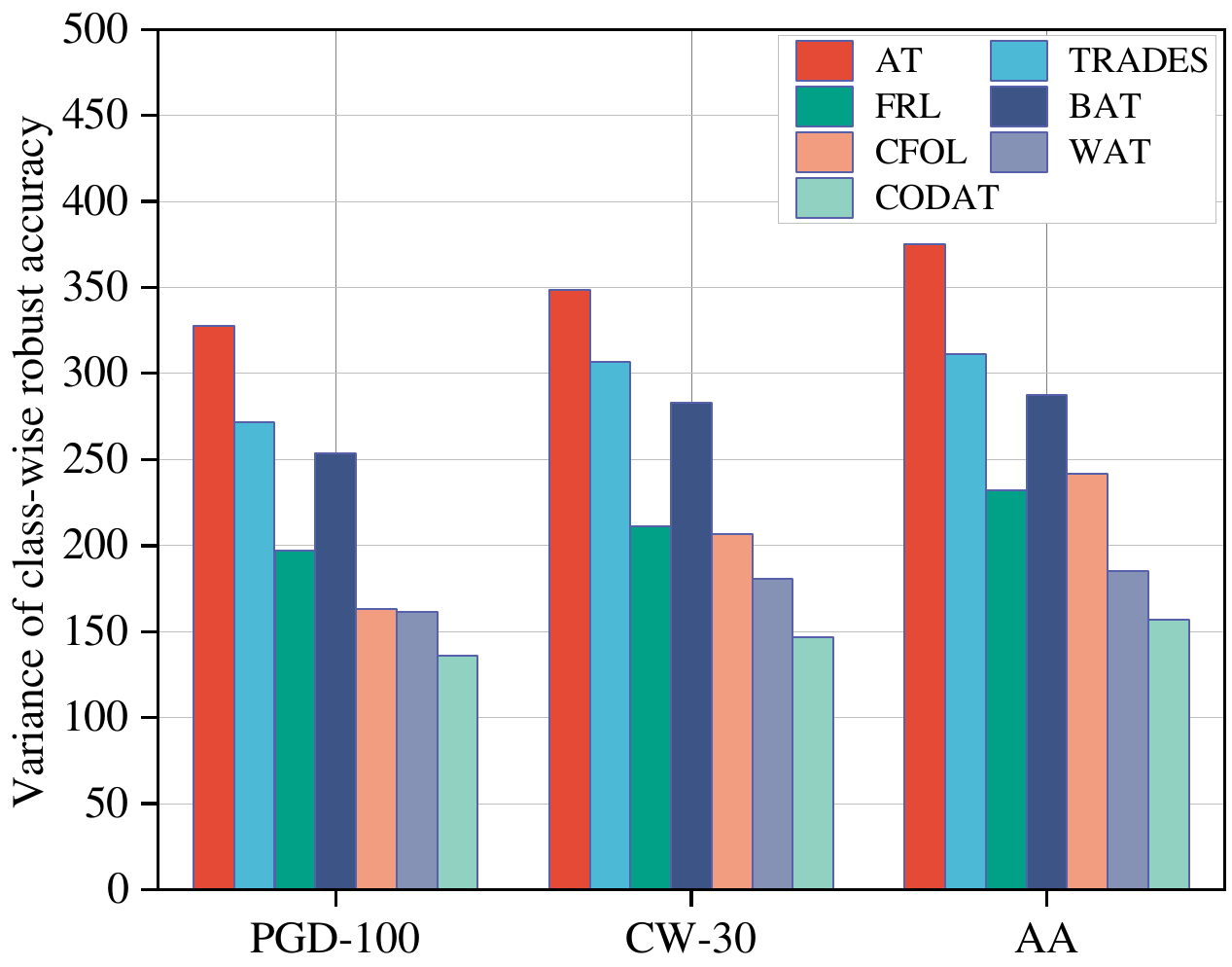}
    \caption{The variance of class-wise robust accuracy for all methods on CIFAR-10 using ResNet-18 under PGD-100, CW-30, and AA attacks, respectively. A lower variance indicates better robust fairness.}
    \label{fig:5}
\end{figure}

\begin{figure}
    \centering
    \includegraphics[width=0.9\linewidth]{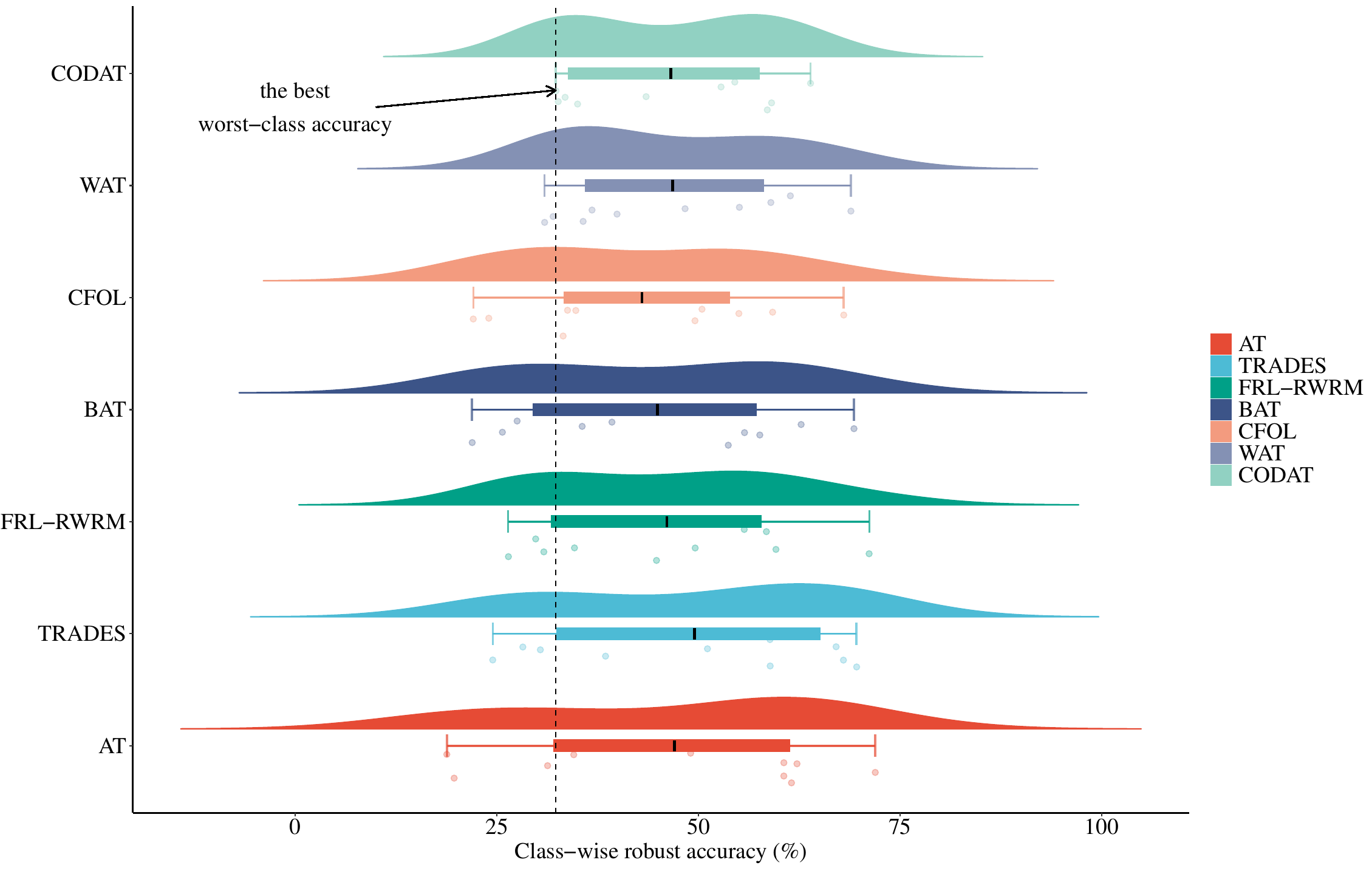}
    \caption{Distribution of class-wise robust accuracy for all methods on CIFAR-10 using ResNet-18 under AA attack.}
    \label{fig:6}
\end{figure}

\begin{figure*}
    \centering
    \subfloat[Variations in average robust accuracy and worst-class robust accuracy under PGD-100 attack.]{
    \includegraphics[width=.45\linewidth]{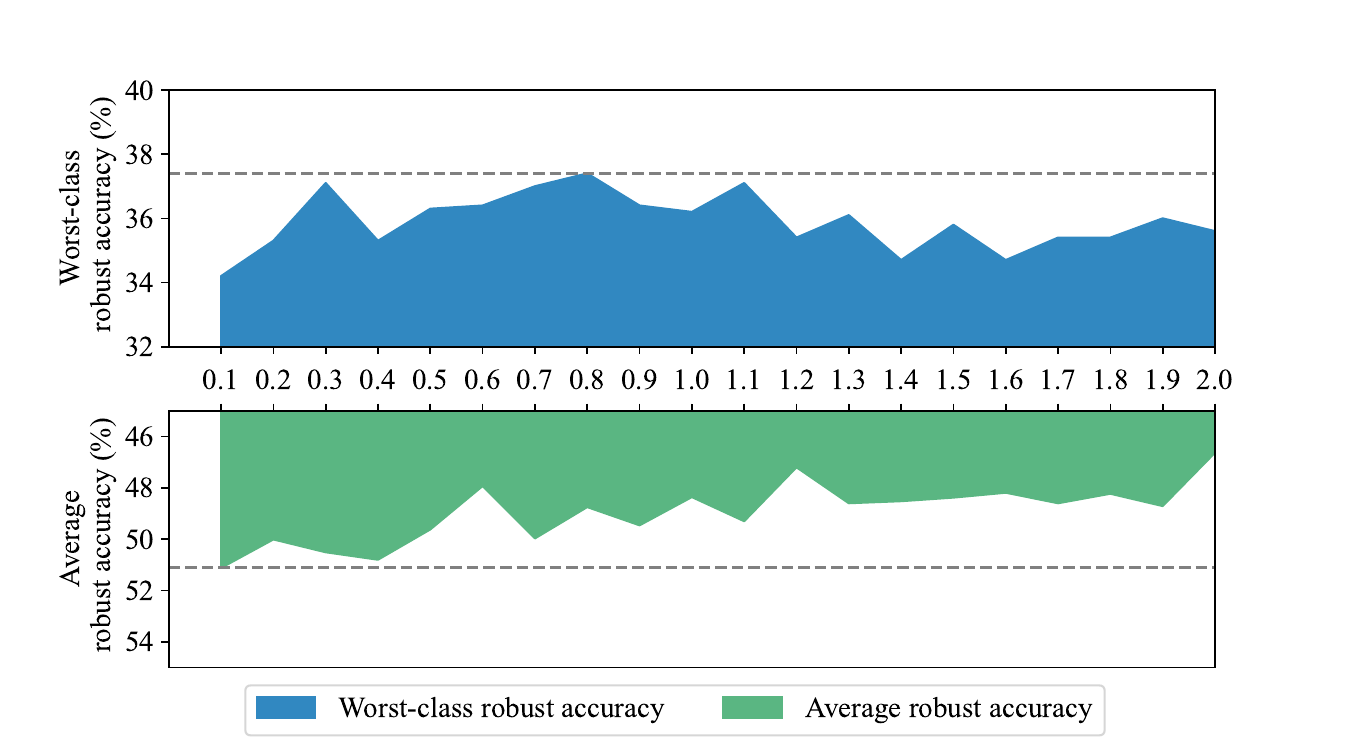}
    \label{fig:7-1}
    }
    \hfil
    \subfloat[Variations in average robust accuracy and worst-class robust accuracy under CW-30 attack.]{
    \includegraphics[width=.45\linewidth]{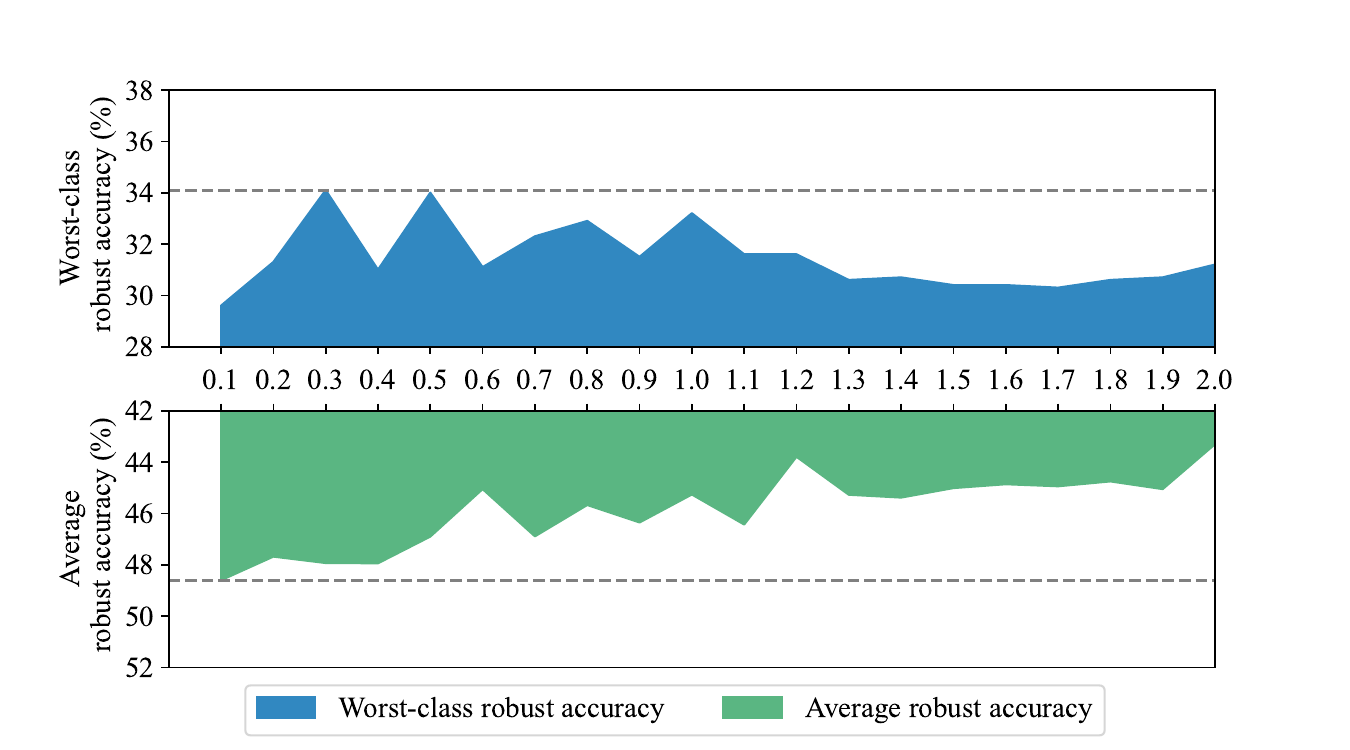}
    \label{fig:7-2}
    }
    \caption{The influence of hyper-parameter $\eta$ on the average robust accuracy and worst-class robust accuracy on CIFAR-10 using ResNet-18 under PGD-100 and CW-30 attacks.}
    \label{fig:7}
\end{figure*}

\subsection{Comparison of the class-wise robust accuracy distribution}\label{sec:5-4}
Furthermore, we conduct a vertical comparative analysis of the class-wise robust accuracy distribution for all methods on CIFAR-10 using ResNet-18 under AA attack, as depicted in~\cref{fig:6}. Firstly, it is evident that CODAT has the smallest range among all methods, indicating a more concentrated distribution of robust accuracies across classes. This is consistent with the findings presented in~\cref{fig:4} and~\cref{fig:5}. Subsequently, the class-wise robust accuracy for all methods exhibits a bimodal distribution pattern. In particular, AT exhibits a higher peak at approximately 60\% and a lower peak at around 23\%, while TRADES displays a higher peak near 64\% and a lower peak near 27\%. The substantial separation between these two peaks results in a diminished level of robust fairness for AT and TRADES, despite their high robustness. In contrast, the intervals between the two peaks are significantly narrower for the other methods, with CODAT exhibiting the smallest interval and the highest peak. This suggests that CODAT may exhibit a slight decline in robust accuracy for the best-performing classes, but substantially improves robust accuracy for the worst-performing classes. 

Additionally, \cref{fig:6} illustrates that, although CODAT's average robust accuracy is slightly lower than that of TRADES, it is generally comparable to that of other methods. This indicates that CODAT can effectively allocate the limited model capacity, maintaining average robust accuracy while significantly improving the robust accuracy of the worst-performing classes, thereby enhancing the model's robust fairness.

\subsection{Hyper-parameter sensitivity analysis on \texorpdfstring{$\eta$}{}}\label{sec:5-5}
In the field of distributionally robust optimization, the constraint of the ambiguity set plays a pivotal role. We investigate the influence of the hyper-parameter $\eta$ within the CODAT framework on both the average robust accuracy and the worst-class robust accuracy through numerical experiments on CIFAR-10 using ResNet-18. Specifically, we train the model with $\eta\in\{0.1, 0.2, \cdots, 2.0\}$, and then evaluate its average robust accuracy and worst-case class robust accuracy under PGD-100 and CW-30 attacks, as shown in~\cref{fig:7}.

\cref{fig:7} illustrates that the worst-class robust accuracy initially increases and then decreases as the hyper-parameter $\eta$ increases. This trend may be attributed to the limited exploration space of the CODAT framework at lower values of $\eta$, which leads to a smaller ambiguity set. Consequently, the model fails to learn an optimal adversarial distribution, resulting in suboptimal performance. As $\eta$ increases, the exploration space broadens, providing the model with a greater opportunity to learn a superior distribution and thus improving its performance. However, an excessively large $\eta$ may result in an overly broad exploration space, which may cause the model to become overly conservative and lead to a decline in performance, potentially culminating in training failure. In contrast, the average robust accuracy demonstrates a consistent decline as $\eta$ increases.

Moreover, during the initial and intermediate stages of the increase in $\eta$, a trade-off is observable between the average robust accuracy and the worst-class robust accuracy. The improvement in the worst-class robust accuracy is accompanied by a reduction in the average robust accuracy. In the later stages, the conservative model results in a decline in both the robust accuracy and the average robust accuracy. To balance the robust accuracy and the average robust accuracy in our experiments, we determine the optimal hyper-parameter $\eta$ settings for different models. In particular, when employing ResNet-18 model, $\eta$ is optimally set to 0.3. Conversely, when utilizing WideResNet-34-10 model, the optimal setting for $\eta$ is 0.8.

\section{Conclusion}\label{sec:6}
To enhance the robust fairness of adversarial training, this paper proposes a novel class optimal distribution adversarial training framework which is grounded in the principles of distributionally robust optimization. We derive a closed-form optimal solution to the internal maximization problem, establishing a theoretical foundation for the joint optimization of class weights and model parameters. Moreover, we propose a fairness elasticity coefficient to assess the algorithm's performance, taking into account both the average accuracy and the worst-class robust accuracy. The experimental results demonstrate the effectiveness of our approach in enhancing robust fairness.

In the future, we intend to investigate the underlying causes of the bimodal distribution phenomenon observed in the class-wise robust accuracy within adversarial training and to develop strategies to further enhance robust fairness.

\section*{Acknowledgements}
This work is supported by the Program of Song Shan Laboratory (Included in the management of the Major Science and Technology Program of Henan Province) (2211002-10700-03).

% \section*{Declarations}

% \begin{itemize}
% \item Competing interests:  The authors declare that they have no known competing financial interests or personal relationships that could have appeared to influence the work reported in this paper.
% \item Data availability: The datasets used in this study are publicly available and can be accessed from \url{https://www.cs.toronto.edu/~kriz/cifar.html}, \url{http://ufldl.stanford.edu/housenumbers/} and \url{https://cs.stanford.edu/~acoates/stl10/}.
% \end{itemize}

\bigskip
\bibliography{fairness}


\begin{thebibliography}{39}
\ifx \bisbn   \undefined \def \bisbn  #1{ISBN #1}\fi
\ifx \binits  \undefined \def \binits#1{#1}\fi
\ifx \bauthor  \undefined \def \bauthor#1{#1}\fi
\ifx \batitle  \undefined \def \batitle#1{#1}\fi
\ifx \bjtitle  \undefined \def \bjtitle#1{#1}\fi
\ifx \bvolume  \undefined \def \bvolume#1{\textbf{#1}}\fi
\ifx \byear  \undefined \def \byear#1{#1}\fi
\ifx \bissue  \undefined \def \bissue#1{#1}\fi
\ifx \bfpage  \undefined \def \bfpage#1{#1}\fi
\ifx \blpage  \undefined \def \blpage #1{#1}\fi
\ifx \burl  \undefined \def \burl#1{\textsf{#1}}\fi
\ifx \doiurl  \undefined \def \doiurl#1{\url{https://doi.org/#1}}\fi
\ifx \betal  \undefined \def \betal{\textit{et al.}}\fi
\ifx \binstitute  \undefined \def \binstitute#1{#1}\fi
\ifx \binstitutionaled  \undefined \def \binstitutionaled#1{#1}\fi
\ifx \bctitle  \undefined \def \bctitle#1{#1}\fi
\ifx \beditor  \undefined \def \beditor#1{#1}\fi
\ifx \bpublisher  \undefined \def \bpublisher#1{#1}\fi
\ifx \bbtitle  \undefined \def \bbtitle#1{#1}\fi
\ifx \bedition  \undefined \def \bedition#1{#1}\fi
\ifx \bseriesno  \undefined \def \bseriesno#1{#1}\fi
\ifx \blocation  \undefined \def \blocation#1{#1}\fi
\ifx \bsertitle  \undefined \def \bsertitle#1{#1}\fi
\ifx \bsnm \undefined \def \bsnm#1{#1}\fi
\ifx \bsuffix \undefined \def \bsuffix#1{#1}\fi
\ifx \bparticle \undefined \def \bparticle#1{#1}\fi
\ifx \barticle \undefined \def \barticle#1{#1}\fi
\bibcommenthead
\ifx \bconfdate \undefined \def \bconfdate #1{#1}\fi
\ifx \botherref \undefined \def \botherref #1{#1}\fi
\ifx \url \undefined \def \url#1{\textsf{#1}}\fi
\ifx \bchapter \undefined \def \bchapter#1{#1}\fi
\ifx \bbook \undefined \def \bbook#1{#1}\fi
\ifx \bcomment \undefined \def \bcomment#1{#1}\fi
\ifx \oauthor \undefined \def \oauthor#1{#1}\fi
\ifx \citeauthoryear \undefined \def \citeauthoryear#1{#1}\fi
\ifx \endbibitem  \undefined \def \endbibitem {}\fi
\ifx \bconflocation  \undefined \def \bconflocation#1{#1}\fi
\ifx \arxivurl  \undefined \def \arxivurl#1{\textsf{#1}}\fi
\csname PreBibitemsHook\endcsname

%%% 1
\bibitem[\protect\citeauthoryear{{Goodfellow} et~al.}{2014}]{2014arXiv1412.6572G}
\begin{botherref}
\oauthor{\bsnm{{Goodfellow}}, \binits{I.J.}},
\oauthor{\bsnm{{Shlens}}, \binits{J.}},
\oauthor{\bsnm{{Szegedy}}, \binits{C.}}:
{Explaining and Harnessing Adversarial Examples}.
arXiv e-prints,
1412--6572
(2014)
\doiurl{10.48550/arXiv.1412.6572}
{\href{https://arxiv.org/abs/1412.6572}{{arXiv:1412.6572}}}
{[stat.ML]}
\end{botherref}
\endbibitem

%%% 2
\bibitem[\protect\citeauthoryear{{Szegedy} et~al.}{2013}]{2013arXiv1312.6199S}
\begin{botherref}
\oauthor{\bsnm{{Szegedy}}, \binits{C.}},
\oauthor{\bsnm{{Zaremba}}, \binits{W.}},
\oauthor{\bsnm{{Sutskever}}, \binits{I.}},
\oauthor{\bsnm{{Bruna}}, \binits{J.}},
\oauthor{\bsnm{{Erhan}}, \binits{D.}},
\oauthor{\bsnm{{Goodfellow}}, \binits{I.}},
\oauthor{\bsnm{{Fergus}}, \binits{R.}}:
{Intriguing properties of neural networks}.
arXiv e-prints,
1312--6199
(2013)
\doiurl{10.48550/arXiv.1312.6199}
{\href{https://arxiv.org/abs/1312.6199}{{arXiv:1312.6199}}}
{[cs.CV]}
\end{botherref}
\endbibitem

%%% 3
\bibitem[\protect\citeauthoryear{Chen et~al.}{2015}]{Chen_Seff_Kornhauser_Xiao_2015}
\begin{bchapter}
\bauthor{\bsnm{Chen}, \binits{C.}},
\bauthor{\bsnm{Seff}, \binits{A.}},
\bauthor{\bsnm{Kornhauser}, \binits{A.}},
\bauthor{\bsnm{Xiao}, \binits{J.}}:
\bctitle{Deepdriving: Learning affordance for direct perception in autonomous driving}.
In: \bbtitle{2015 IEEE International Conference on Computer Vision (ICCV)}
(\byear{2015}).
\doiurl{10.1109/iccv.2015.312} .
\burl{http://dx.doi.org/10.1109/iccv.2015.312}
\end{bchapter}
\endbibitem

%%% 4
\bibitem[\protect\citeauthoryear{Ma et~al.}{2021}]{Ma_Niu_Gu_Wang_Zhao_Bailey_Lu_2021}
\begin{botherref}
\oauthor{\bsnm{Ma}, \binits{X.}},
\oauthor{\bsnm{Niu}, \binits{Y.}},
\oauthor{\bsnm{Gu}, \binits{L.}},
\oauthor{\bsnm{Wang}, \binits{Y.}},
\oauthor{\bsnm{Zhao}, \binits{Y.}},
\oauthor{\bsnm{Bailey}, \binits{J.}},
\oauthor{\bsnm{Lu}, \binits{F.}}:
Understanding adversarial attacks on deep learning based medical image analysis systems.
Pattern Recognition,
107332
(2021)
\doiurl{10.1016/j.patcog.2020.107332}
\end{botherref}
\endbibitem

%%% 5
\bibitem[\protect\citeauthoryear{{Morgulis} et~al.}{2019}]{2019arXiv190700374M}
\begin{botherref}
\oauthor{\bsnm{{Morgulis}}, \binits{N.}},
\oauthor{\bsnm{{Kreines}}, \binits{A.}},
\oauthor{\bsnm{{Mendelowitz}}, \binits{S.}},
\oauthor{\bsnm{{Weisglass}}, \binits{Y.}}:
{Fooling a Real Car with Adversarial Traffic Signs}.
arXiv e-prints,
1907--00374
(2019)
\doiurl{10.48550/arXiv.1907.00374}
{\href{https://arxiv.org/abs/1907.00374}{{arXiv:1907.00374}}}
{[cs.CR]}
\end{botherref}
\endbibitem

%%% 6
\bibitem[\protect\citeauthoryear{Sharif et~al.}{2016}]{Sharif_Bhagavatula_Bauer_Reiter_2016}
\begin{bchapter}
\bauthor{\bsnm{Sharif}, \binits{M.}},
\bauthor{\bsnm{Bhagavatula}, \binits{S.}},
\bauthor{\bsnm{Bauer}, \binits{L.}},
\bauthor{\bsnm{Reiter}, \binits{M.K.}}:
\bctitle{Accessorize to a crime}.
In: \bbtitle{Proceedings of the 2016 ACM SIGSAC Conference on Computer and Communications Security}
(\byear{2016}).
\doiurl{10.1145/2976749.2978392} .
\burl{http://dx.doi.org/10.1145/2976749.2978392}
\end{bchapter}
\endbibitem

%%% 7
\bibitem[\protect\citeauthoryear{{Madry} et~al.}{2017}]{2017arXiv170606083M}
\begin{botherref}
\oauthor{\bsnm{{Madry}}, \binits{A.}},
\oauthor{\bsnm{{Makelov}}, \binits{A.}},
\oauthor{\bsnm{{Schmidt}}, \binits{L.}},
\oauthor{\bsnm{{Tsipras}}, \binits{D.}},
\oauthor{\bsnm{{Vladu}}, \binits{A.}}:
{Towards Deep Learning Models Resistant to Adversarial Attacks}.
arXiv e-prints,
1706--06083
(2017)
\doiurl{10.48550/arXiv.1706.06083}
{\href{https://arxiv.org/abs/1706.06083}{{arXiv:1706.06083}}}
{[stat.ML]}
\end{botherref}
\endbibitem

%%% 8
\bibitem[\protect\citeauthoryear{Papernot et~al.}{2016}]{Papernot_McDaniel_Wu_Jha_Swami_2016}
\begin{bchapter}
\bauthor{\bsnm{Papernot}, \binits{N.}},
\bauthor{\bsnm{McDaniel}, \binits{P.}},
\bauthor{\bsnm{Wu}, \binits{X.}},
\bauthor{\bsnm{Jha}, \binits{S.}},
\bauthor{\bsnm{Swami}, \binits{A.}}:
\bctitle{Distillation as a defense to adversarial perturbations against deep neural networks}.
In: \bbtitle{2016 IEEE Symposium on Security and Privacy (SP)}
(\byear{2016}).
\doiurl{10.1109/sp.2016.41} .
\burl{http://dx.doi.org/10.1109/sp.2016.41}
\end{bchapter}
\endbibitem

%%% 9
\bibitem[\protect\citeauthoryear{Raghunathan et~al.}{2018}]{Raghunathan_Steinhardt_Liang_2018}
\begin{botherref}
\oauthor{\bsnm{Raghunathan}, \binits{A.}},
\oauthor{\bsnm{Steinhardt}, \binits{J.}},
\oauthor{\bsnm{Liang}, \binits{P.}}:
Certified defenses against adversarial examples.
International Conference on Learning Representations,International Conference on Learning Representations
(2018)
\end{botherref}
\endbibitem

%%% 10
\bibitem[\protect\citeauthoryear{Xie et~al.}{2019}]{Xie_Wu_Maaten_Yuille_He_2019}
\begin{bchapter}
\bauthor{\bsnm{Xie}, \binits{C.}},
\bauthor{\bsnm{Wu}, \binits{Y.}},
\bauthor{\bsnm{Maaten}, \binits{L.v.d.}},
\bauthor{\bsnm{Yuille}, \binits{A.L.}},
\bauthor{\bsnm{He}, \binits{K.}}:
\bctitle{Feature denoising for improving adversarial robustness}.
In: \bbtitle{2019 IEEE/CVF Conference on Computer Vision and Pattern Recognition (CVPR)}
(\byear{2019}).
\doiurl{10.1109/cvpr.2019.00059} .
\burl{http://dx.doi.org/10.1109/cvpr.2019.00059}
\end{bchapter}
\endbibitem

%%% 11
\bibitem[\protect\citeauthoryear{Xu et~al.}{2018}]{Xu_Evans_Qi_2018}
\begin{bchapter}
\bauthor{\bsnm{Xu}, \binits{W.}},
\bauthor{\bsnm{Evans}, \binits{D.}},
\bauthor{\bsnm{Qi}, \binits{Y.}}:
\bctitle{Feature squeezing: Detecting adversarial examples in deep neural networks}.
In: \bbtitle{Proceedings 2018 Network and Distributed System Security Symposium}
(\byear{2018}).
\doiurl{10.14722/ndss.2018.23198} .
\burl{http://dx.doi.org/10.14722/ndss.2018.23198}
\end{bchapter}
\endbibitem

%%% 12
\bibitem[\protect\citeauthoryear{Zhi et~al.}{2024}]{zhi2024ma}
\begin{botherref}
\oauthor{\bsnm{Zhi}, \binits{H.}},
\oauthor{\bsnm{Yu}, \binits{H.}},
\oauthor{\bsnm{Li}, \binits{S.}},
\oauthor{\bsnm{Huang}, \binits{R.}}:
Ma-cat: Misclassification-aware contrastive adversarial training.
Advanced Intelligent Systems,
2300658
(2024)
\end{botherref}
\endbibitem

%%% 13
\bibitem[\protect\citeauthoryear{Zhang et~al.}{2019}]{Zhang_Yu_Jiao_Xing_Ghaoui_Jordan_2019}
\begin{botherref}
\oauthor{\bsnm{Zhang}, \binits{H.}},
\oauthor{\bsnm{Yu}, \binits{Y.}},
\oauthor{\bsnm{Jiao}, \binits{J.}},
\oauthor{\bsnm{Xing}, \binits{E.}},
\oauthor{\bsnm{Ghaoui}, \binits{L.}},
\oauthor{\bsnm{Jordan}, \binits{M.}}:
Theoretically principled trade-off between robustness and accuracy.
International Conference on Machine Learning,International Conference on Machine Learning
(2019)
\end{botherref}
\endbibitem

%%% 14
\bibitem[\protect\citeauthoryear{Benz et~al.}{2021}]{benz2021robustness}
\begin{bchapter}
\bauthor{\bsnm{Benz}, \binits{P.}},
\bauthor{\bsnm{Zhang}, \binits{C.}},
\bauthor{\bsnm{Karjauv}, \binits{A.}},
\bauthor{\bsnm{Kweon}, \binits{I.S.}}:
\bctitle{Robustness may be at odds with fairness: An empirical study on class-wise accuracy}.
In: \bbtitle{NeurIPS 2020 Workshop on Pre-registration in Machine Learning},
pp. \bfpage{325}--\blpage{342}
(\byear{2021}).
\bcomment{PMLR}
\end{bchapter}
\endbibitem

%%% 15
\bibitem[\protect\citeauthoryear{{Tian} et~al.}{2021}]{2021arXiv210514240T}
\begin{botherref}
\oauthor{\bsnm{{Tian}}, \binits{Q.}},
\oauthor{\bsnm{{Kuang}}, \binits{K.}},
\oauthor{\bsnm{{Jiang}}, \binits{K.}},
\oauthor{\bsnm{{Wu}}, \binits{F.}},
\oauthor{\bsnm{{Wang}}, \binits{Y.}}:
{Analysis and Applications of Class-wise Robustness in Adversarial Training}.
arXiv e-prints,
2105--14240
(2021)
\doiurl{10.48550/arXiv.2105.14240}
{\href{https://arxiv.org/abs/2105.14240}{{arXiv:2105.14240}}}
{[cs.CV]}
\end{botherref}
\endbibitem

%%% 16
\bibitem[\protect\citeauthoryear{Xu et~al.}{2021}]{xu2021robust}
\begin{bchapter}
\bauthor{\bsnm{Xu}, \binits{H.}},
\bauthor{\bsnm{Liu}, \binits{X.}},
\bauthor{\bsnm{Li}, \binits{Y.}},
\bauthor{\bsnm{Jain}, \binits{A.}},
\bauthor{\bsnm{Tang}, \binits{J.}}:
\bctitle{To be robust or to be fair: Towards fairness in adversarial training}.
In: \bbtitle{International Conference on Machine Learning},
pp. \bfpage{11492}--\blpage{11501}
(\byear{2021}).
\bcomment{PMLR}
\end{bchapter}
\endbibitem

%%% 17
\bibitem[\protect\citeauthoryear{Li and Liu}{2023}]{li2023wat}
\begin{bchapter}
\bauthor{\bsnm{Li}, \binits{B.}},
\bauthor{\bsnm{Liu}, \binits{W.}}:
\bctitle{Wat: improve the worst-class robustness in adversarial training}.
In: \bbtitle{Proceedings of the AAAI Conference on Artificial Intelligence},
vol. \bseriesno{37},
pp. \bfpage{14982}--\blpage{14990}
(\byear{2023})
\end{bchapter}
\endbibitem

%%% 18
\bibitem[\protect\citeauthoryear{{Pethick} et~al.}{2023}]{2023arXiv230208872P}
\begin{botherref}
\oauthor{\bsnm{{Pethick}}, \binits{T.}},
\oauthor{\bsnm{{Chrysos}}, \binits{G.G.}},
\oauthor{\bsnm{{Cevher}}, \binits{V.}}:
{Revisiting adversarial training for the worst-performing class}.
arXiv e-prints,
2302--08872
(2023)
\doiurl{10.48550/arXiv.2302.08872}
{\href{https://arxiv.org/abs/2302.08872}{{arXiv:2302.08872}}}
{[cs.LG]}
\end{botherref}
\endbibitem

%%% 19
\bibitem[\protect\citeauthoryear{Sun et~al.}{2023}]{sun2023improving}
\begin{bchapter}
\bauthor{\bsnm{Sun}, \binits{C.}},
\bauthor{\bsnm{Xu}, \binits{C.}},
\bauthor{\bsnm{Yao}, \binits{C.}},
\bauthor{\bsnm{Liang}, \binits{S.}},
\bauthor{\bsnm{Wu}, \binits{Y.}},
\bauthor{\bsnm{Liang}, \binits{D.}},
\bauthor{\bsnm{Liu}, \binits{X.}},
\bauthor{\bsnm{Liu}, \binits{A.}}:
\bctitle{Improving robust fariness via balance adversarial training}.
In: \bbtitle{Proceedings of the AAAI Conference on Artificial Intelligence},
vol. \bseriesno{37},
pp. \bfpage{15161}--\blpage{15169}
(\byear{2023})
\end{bchapter}
\endbibitem

%%% 20
\bibitem[\protect\citeauthoryear{Wei et~al.}{2023}]{10205260}
\begin{bchapter}
\bauthor{\bsnm{Wei}, \binits{Z.}},
\bauthor{\bsnm{Wang}, \binits{Y.}},
\bauthor{\bsnm{Guo}, \binits{Y.}},
\bauthor{\bsnm{Wang}, \binits{Y.}}:
\bctitle{Cfa: Class-wise calibrated fair adversarial training}.
In: \bbtitle{2023 IEEE/CVF Conference on Computer Vision and Pattern Recognition (CVPR)},
pp. \bfpage{8193}--\blpage{8201}
(\byear{2023}).
\doiurl{10.1109/CVPR52729.2023.00792}
\end{bchapter}
\endbibitem

%%% 21
\bibitem[\protect\citeauthoryear{{Rahimian} and {Mehrotra}}{2019}]{2019arXiv190805659R}
\begin{botherref}
\oauthor{\bsnm{{Rahimian}}, \binits{H.}},
\oauthor{\bsnm{{Mehrotra}}, \binits{S.}}:
{Distributionally Robust Optimization: A Review}.
arXiv e-prints,
1908--05659
(2019)
\doiurl{10.48550/arXiv.1908.05659}
{\href{https://arxiv.org/abs/1908.05659}{{arXiv:1908.05659}}}
{[math.OC]}
\end{botherref}
\endbibitem

%%% 22
\bibitem[\protect\citeauthoryear{Tsipras et~al.}{2018}]{Tsipras_Santurkar_Engstrom_Turner_Madry_2018}
\begin{botherref}
\oauthor{\bsnm{Tsipras}, \binits{D.}},
\oauthor{\bsnm{Santurkar}, \binits{S.}},
\oauthor{\bsnm{Engstrom}, \binits{L.}},
\oauthor{\bsnm{Turner}, \binits{A.}},
\oauthor{\bsnm{Madry}, \binits{A.}}:
Robustness may be at odds with accuracy.
International Conference on Learning Representations,International Conference on Learning Representations
(2018)
\end{botherref}
\endbibitem

%%% 23
\bibitem[\protect\citeauthoryear{Yang et~al.}{2020}]{Yang_Rashtchian_Zhang_Salakhutdinov_Chaudhuri_2020}
\begin{botherref}
\oauthor{\bsnm{Yang}, \binits{Y.-Y.}},
\oauthor{\bsnm{Rashtchian}, \binits{C.}},
\oauthor{\bsnm{Zhang}, \binits{H.}},
\oauthor{\bsnm{Salakhutdinov}, \binits{R.}},
\oauthor{\bsnm{Chaudhuri}, \binits{K.}}:
A closer look at accuracy vs. robustness.
Neural Information Processing Systems,Neural Information Processing Systems
(2020)
\end{botherref}
\endbibitem

%%% 24
\bibitem[\protect\citeauthoryear{{Rice} et~al.}{2020}]{2020arXiv200211569R}
\begin{botherref}
\oauthor{\bsnm{{Rice}}, \binits{L.}},
\oauthor{\bsnm{{Wong}}, \binits{E.}},
\oauthor{\bsnm{{Zico Kolter}}, \binits{J.}}:
{Overfitting in adversarially robust deep learning}.
arXiv e-prints,
2002--11569
(2020)
\doiurl{10.48550/arXiv.2002.11569}
{\href{https://arxiv.org/abs/2002.11569}{{arXiv:2002.11569}}}
{[cs.LG]}
\end{botherref}
\endbibitem

%%% 25
\bibitem[\protect\citeauthoryear{Freund and Schapire}{1997}]{Freund_Schapire_1997}
\begin{botherref}
\oauthor{\bsnm{Freund}, \binits{Y.}},
\oauthor{\bsnm{Schapire}, \binits{R.E.}}:
A decision-theoretic generalization of on-line learning and an application to boosting.
Journal of Computer and System Sciences,
119--139
(1997)
\doiurl{10.1006/jcss.1997.1504}
\end{botherref}
\endbibitem

%%% 26
\bibitem[\protect\citeauthoryear{Zhang et~al.}{2024}]{zhang2024towards}
\begin{bchapter}
\bauthor{\bsnm{Zhang}, \binits{Y.}},
\bauthor{\bsnm{Zhang}, \binits{T.}},
\bauthor{\bsnm{Mu}, \binits{R.}},
\bauthor{\bsnm{Huang}, \binits{X.}},
\bauthor{\bsnm{Ruan}, \binits{W.}}:
\bctitle{Towards fairness-aware adversarial learning}.
In: \bbtitle{Proceedings of the IEEE/CVF Conference on Computer Vision and Pattern Recognition},
pp. \bfpage{24746}--\blpage{24755}
(\byear{2024})
\end{bchapter}
\endbibitem

%%% 27
\bibitem[\protect\citeauthoryear{Rockafellar and Uryasev}{2016}]{Rockafellar_Uryasev_2016}
\begin{botherref}
\oauthor{\bsnm{Rockafellar}, \binits{R.T.}},
\oauthor{\bsnm{Uryasev}, \binits{S.}}:
Optimization of conditional value-at-risk.
The Journal of Risk,
21--41
(2016)
\doiurl{10.21314/jor.2000.038}
\end{botherref}
\endbibitem

%%% 28
\bibitem[\protect\citeauthoryear{Auer et~al.}{2002}]{Auer_Cesa-Bianchi_Freund_Schapire_2002}
\begin{botherref}
\oauthor{\bsnm{Auer}, \binits{P.}},
\oauthor{\bsnm{Cesa-Bianchi}, \binits{N.}},
\oauthor{\bsnm{Freund}, \binits{Y.}},
\oauthor{\bsnm{Schapire}, \binits{R.E.}}:
The nonstochastic multiarmed bandit problem.
SIAM Journal on Computing,
48--77
(2002)
\doiurl{10.1137/s0097539701398375}
\end{botherref}
\endbibitem

%%% 29
\bibitem[\protect\citeauthoryear{{McInnes} et~al.}{2018}]{2018arXiv180203426M}
\begin{botherref}
\oauthor{\bsnm{{McInnes}}, \binits{L.}},
\oauthor{\bsnm{{Healy}}, \binits{J.}},
\oauthor{\bsnm{{Melville}}, \binits{J.}}:
{UMAP: Uniform Manifold Approximation and Projection for Dimension Reduction}.
arXiv e-prints,
1802--03426
(2018)
\doiurl{10.48550/arXiv.1802.03426}
{\href{https://arxiv.org/abs/1802.03426}{{arXiv:1802.03426}}}
{[stat.ML]}
\end{botherref}
\endbibitem

%%% 30
\bibitem[\protect\citeauthoryear{Geirhos et~al.}{2020}]{Geirhos_Jacobsen_Michaelis_Zemel_Brendel_Bethge_Wichmann_2020}
\begin{botherref}
\oauthor{\bsnm{Geirhos}, \binits{R.}},
\oauthor{\bsnm{Jacobsen}, \binits{J.-H.}},
\oauthor{\bsnm{Michaelis}, \binits{C.}},
\oauthor{\bsnm{Zemel}, \binits{R.}},
\oauthor{\bsnm{Brendel}, \binits{W.}},
\oauthor{\bsnm{Bethge}, \binits{M.}},
\oauthor{\bsnm{Wichmann}, \binits{F.A.}}:
Shortcut learning in deep neural networks.
Nature Machine Intelligence,
665--673
(2020)
\doiurl{10.1038/s42256-020-00257-z}
\end{botherref}
\endbibitem

%%% 31
\bibitem[\protect\citeauthoryear{{Qi} et~al.}{2020}]{2020arXiv200610138Q}
\begin{botherref}
\oauthor{\bsnm{{Qi}}, \binits{Q.}},
\oauthor{\bsnm{{Guo}}, \binits{Z.}},
\oauthor{\bsnm{{Xu}}, \binits{Y.}},
\oauthor{\bsnm{{Jin}}, \binits{R.}},
\oauthor{\bsnm{{Yang}}, \binits{T.}}:
{An Online Method for A Class of Distributionally Robust Optimization with Non-Convex Objectives}.
arXiv e-prints,
2006--10138
(2020)
\doiurl{10.48550/arXiv.2006.10138}
{\href{https://arxiv.org/abs/2006.10138}{{arXiv:2006.10138}}}
{[cs.LG]}
\end{botherref}
\endbibitem

%%% 32
\bibitem[\protect\citeauthoryear{Wang et~al.}{2021}]{wang2021adversarial}
\begin{barticle}
\bauthor{\bsnm{Wang}, \binits{J.}},
\bauthor{\bsnm{Zhang}, \binits{T.}},
\bauthor{\bsnm{Liu}, \binits{S.}},
\bauthor{\bsnm{Chen}, \binits{P.-Y.}},
\bauthor{\bsnm{Xu}, \binits{J.}},
\bauthor{\bsnm{Fardad}, \binits{M.}},
\bauthor{\bsnm{Li}, \binits{B.}}:
\batitle{Adversarial attack generation empowered by min-max optimization}.
\bjtitle{Advances in Neural Information Processing Systems}
\bvolume{34},
\bfpage{16020}--\blpage{16033}
(\byear{2021})
\end{barticle}
\endbibitem

%%% 33
\bibitem[\protect\citeauthoryear{Krizhevsky et~al.}{2009}]{krizhevsky2009learning}
\begin{botherref}
\oauthor{\bsnm{Krizhevsky}, \binits{A.}},
\oauthor{\bsnm{Hinton}, \binits{G.}}, et al.:
Learning multiple layers of features from tiny images
(2009)
\end{botherref}
\endbibitem

%%% 34
\bibitem[\protect\citeauthoryear{Netzer et~al.}{2011}]{netzer2011reading}
\begin{bchapter}
\bauthor{\bsnm{Netzer}, \binits{Y.}},
\bauthor{\bsnm{Wang}, \binits{T.}},
\bauthor{\bsnm{Coates}, \binits{A.}},
\bauthor{\bsnm{Bissacco}, \binits{A.}},
\bauthor{\bsnm{Wu}, \binits{B.}},
\bauthor{\bsnm{Ng}, \binits{A.Y.}}, \betal:
\bctitle{Reading digits in natural images with unsupervised feature learning}.
In: \bbtitle{NIPS Workshop on Deep Learning and Unsupervised Feature Learning},
vol. \bseriesno{2011},
p. \bfpage{7}
(\byear{2011}).
\bcomment{Granada, Spain}
\end{bchapter}
\endbibitem

%%% 35
\bibitem[\protect\citeauthoryear{Coates et~al.}{2011}]{Coates_Ng_Lee_2011}
\begin{botherref}
\oauthor{\bsnm{Coates}, \binits{A.}},
\oauthor{\bsnm{Ng}, \binits{A.}},
\oauthor{\bsnm{Lee}, \binits{H.}}:
An analysis of single-layer networks in unsupervised feature learning.
International Conference on Artificial Intelligence and Statistics,International Conference on Artificial Intelligence and Statistics
(2011)
\end{botherref}
\endbibitem

%%% 36
\bibitem[\protect\citeauthoryear{He et~al.}{2016}]{He_Zhang_Ren_Sun_2016}
\begin{bchapter}
\bauthor{\bsnm{He}, \binits{K.}},
\bauthor{\bsnm{Zhang}, \binits{X.}},
\bauthor{\bsnm{Ren}, \binits{S.}},
\bauthor{\bsnm{Sun}, \binits{J.}}:
\bctitle{Deep residual learning for image recognition}.
In: \bbtitle{2016 IEEE Conference on Computer Vision and Pattern Recognition (CVPR)}
(\byear{2016}).
\doiurl{10.1109/cvpr.2016.90} .
\burl{http://dx.doi.org/10.1109/cvpr.2016.90}
\end{bchapter}
\endbibitem

%%% 37
\bibitem[\protect\citeauthoryear{Zagoruyko and Komodakis}{2016}]{Zagoruyko_Komodakis_2016}
\begin{bchapter}
\bauthor{\bsnm{Zagoruyko}, \binits{S.}},
\bauthor{\bsnm{Komodakis}, \binits{N.}}:
\bctitle{Wide residual networks}.
In: \bbtitle{Procedings of the British Machine Vision Conference 2016}
(\byear{2016}).
\doiurl{10.5244/c.30.87} .
\burl{http://dx.doi.org/10.5244/c.30.87}
\end{bchapter}
\endbibitem

%%% 38
\bibitem[\protect\citeauthoryear{Carlini and Wagner}{2017}]{Carlini_Wagner_2017}
\begin{bchapter}
\bauthor{\bsnm{Carlini}, \binits{N.}},
\bauthor{\bsnm{Wagner}, \binits{D.}}:
\bctitle{Towards evaluating the robustness of neural networks}.
In: \bbtitle{2017 IEEE Symposium on Security and Privacy (SP)}
(\byear{2017}).
\doiurl{10.1109/sp.2017.49} .
\burl{http://dx.doi.org/10.1109/sp.2017.49}
\end{bchapter}
\endbibitem

%%% 39
\bibitem[\protect\citeauthoryear{Croce and Hein}{2020}]{Croce_Hein_2020}
\begin{botherref}
\oauthor{\bsnm{Croce}, \binits{F.}},
\oauthor{\bsnm{Hein}, \binits{M.}}:
Reliable evaluation of adversarial robustness with an ensemble of diverse parameter-free attacks.
International Conference on Machine Learning,International Conference on Machine Learning
(2020)
\end{botherref}
\endbibitem

\end{thebibliography}

\clearpage

\begin{appendices}
\onecolumn
\section{Proof of Theorem~\ref{theorem:1}}\label{append:1}
\setcounter{theorem}{0}
\begin{theorem}
    If~\cref{assumption:1} holds true, the closed-form optimal solution to the inner maximization problem within~\cref{eq:5} is
    \begin{equation}
        \begin{split}
        p^*(\xi)=p_0(\xi) 
        +p_0(\xi)\sqrt{\frac{\eta}{{\rm{Var}}_{P_0}[H(\boldsymbol{\theta},\xi)]}}\left\{H(\boldsymbol{\theta},\xi)-\mathbb{E}_{P_0}[H(\boldsymbol{\theta},\xi)]\right\}. \nonumber
        \end{split}
    \end{equation}
\end{theorem}

\begin{proof}
    Let $L(\xi)=\frac{p(\xi)}{p_0(\xi)}$, then $L(\cdot)$ is a Radon-Nikodym derivative. We can easily observe that $L(\xi)\geq 0$ and $\mathbb{E}_{P_0}[L(\xi)]=1$.

    Employing the change-of-measure technique on the objective function of~\cref{eq:5}, we can obtain
    \begin{equation}\label{eq:16}
        \begin{split}
            \mathbb{E}_P[H(\boldsymbol{\theta},\xi)]&=\int_\Xi p(\xi) H(\boldsymbol{\theta}, \xi) d\xi \\
            &= \int_\Xi \frac{p(\xi)}{p_0(\xi)} H(\boldsymbol{\theta},\xi)p_0(\xi)d\xi \\
            &= \int_\Xi L(\xi)H(\boldsymbol{\theta}, \xi) p_0(\xi) d\xi \\
            &=\mathbb{E}_{P_0}[L(\xi)H(\boldsymbol{\theta},\xi)].
        \end{split}
    \end{equation}

    Similarly, we have
    \begin{equation}\label{eq:17}
        \begin{split}
            D(P||P_0)&=\int_\Xi \frac{[p(\xi)-p_0(\xi)]^2}{p_0(\xi)} d\xi \\
            &=\int_\Xi \frac{[p(\xi)-p_0(\xi)]^2}{[p_0(\xi)]^2} p_0(\xi) d\xi \\
            &=\mathbb{E}_{P_0}[(L(\xi)-1)^2].
        \end{split}
    \end{equation}

    Then, we can reformulate the inner maximization problem of~\cref{eq:5} as
    \begin{equation}\label{eq:18}
        \begin{array}{cc}
             \max\limits_{L(\xi)\in\mathbb{L}}\mathbb{E}_{P_0}[L(\xi)H(\boldsymbol{\theta},\xi)],  \\
             \text{s.t.}\quad \mathbb{E}_{P_0}[(L(\xi)-1)^2]\leq \eta,
              
        \end{array}
    \end{equation}
    where $\mathbb{L}=\{L:\mathbb{E}_{P_0}[L(\xi)]=1,L(\xi)\geq 0\}$ denotes the set of Randon-Nikodym derivatives generated by $P$ such that $P$ is absolutely continuous with respect to $P_0$. Therefore, we transform the optimization problem concerning $P$ (the inner maximization problem of~\cref{eq:5}) into one concerning $P_0$ (~\cref{eq:18}).

    ~\cref{eq:18} is a convex optimization problem. Let
    \begin{equation}
        l_0(\alpha, L)=\mathbb{E}_{P_0}[L(\xi)H(\boldsymbol{\theta},\xi)]-\alpha\{\mathbb{E}_{P_0}[(L(\xi)-1)^2]-\eta\} \nonumber
    \end{equation}
    be the Lagrangian function associated to~\cref{eq:18}. Then~\cref{eq:18} is equivalent to
    \begin{equation}
        \label{eq:19}
        \max\limits_{L(\xi)\in \mathbb{L}} \min\limits_{\alpha\geq 0} l_0(\alpha, L).
    \end{equation}

    Swapping the positions of the maximization and minimization operators, we get the Largrangian dual of~\cref{eq:19}
    \begin{equation}
        \label{eq:20}
        \min\limits_{\alpha\geq 0} \max\limits_{L(\xi)\in \mathbb{L}} l_0 (\alpha, L).
    \end{equation}

    The inner maximization problem of~\cref{eq:20} can be formulated as
    \begin{equation}
        \label{eq:21}
        \begin{array}{cc}
             \max\limits_{L(\xi)\in\mathbb{L}_0}\mathbb{E}_{P_0}[L(\xi)H(\boldsymbol{\theta},\xi)-\alpha (L(\xi)-1)^2]+\alpha \eta,  \\
             \text{s.t.}\quad \mathbb{E}_{P_0}[L(\xi)] = 1, 
        \end{array}
    \end{equation}
    where we define $\mathbb{L}_0=\{L\in\mathbb{L}(P_0):L\geq 0,\text{a.s.}\}$.

    ~\cref{eq:21} is also a convex optimization problem. By employing Lagrange multiplier method, we get
    \begin{equation}
        \label{eq:22}
        \begin{split}
        J(L(\xi),\lambda)=\mathbb{E}_{P_0} [L(\xi)H(\boldsymbol{\theta},\xi)-\alpha(L(\xi)-1)^2]+\alpha\eta 
        +\lambda (\mathbb{E}_{P_0}[L(\xi)]-1).
        \end{split}
    \end{equation}

    According to $\nabla_L J(L(\xi),\lambda)=0$ and $\nabla_\lambda J(L(\xi),\lambda)=0$, for $\forall u,v \in [K]$, we have
    \begin{equation}
        \label{eq:23}
        H(\boldsymbol{\theta},u)-2\alpha L(u)=H(\boldsymbol{\theta},v)-2\alpha L(v).
    \end{equation}

    Therefore, we can obtain
    \begin{equation}
        \label{eq:24}
        L(u) = L(v) +\frac{1}{2\alpha}[H(\boldsymbol{\theta},u)-H(\boldsymbol{\theta},v)].
    \end{equation}

    Multiplying both sides of~\cref{eq:24} by $p_0(u)$, we get
    \begin{equation}
        \label{eq:25}
        p_0(u)L(u)=p_0(u)L(v)+\frac{p_0(u)}{2\alpha}[H(\boldsymbol{\theta},u)-H(\boldsymbol{\theta},v)].
    \end{equation}

    Summing over $u\in [K]$, we can yield
    \begin{equation}
        \label{eq:26}
        \begin{split}
        \sum\limits_{u}p_0(u)L(u) =L(v)\sum\limits_u p_0(u) 
        +\frac{1}{2\alpha}\sum\limits_u p_0(u)H(\boldsymbol{\theta},u) 
        -\frac{1}{2\alpha}H(\boldsymbol{\theta},v)\sum\limits_u p_0(u).
        \end{split}
    \end{equation}

    From $\mathbb{E}_{P_0}[L(\xi)]=1$, we can get
    \begin{equation}
        \label{eq:27}
        L(v)=1+\frac{H(\boldsymbol{\theta},v)-\mathbb{E}_{P_0}[H(\boldsymbol{\theta},u)]}{2\alpha}.
    \end{equation}

    The closed-form optimal solution to the inner maximization problem of~\cref{eq:20} is
    \begin{equation}
        \label{eq:28}
        L^*(\xi) = 1+\frac{1}{2\alpha}[H(\boldsymbol{\theta},\xi)-\mathbb{E}_{P_0}[H(\boldsymbol{\theta},\xi)]].
    \end{equation}

    Substituting $L^*(\xi)$ into~\cref{eq:20}, we obtain
    \begin{equation}
        \label{eq:29}
        \min\limits_{\alpha\geq 0}\mathbb{E}_{P_0}[H(\boldsymbol{\theta},\xi)]+\frac{\text{Var}_{P_0}[H(\boldsymbol{\theta},\xi)]}{4\alpha}+\alpha \eta.
    \end{equation}

    \cref{eq:29} is a minimization problem concerning $\alpha$. Let
    \begin{equation}
        \label{eq:30}
        g(\alpha) = \mathbb{E}_{P_0}[H(\boldsymbol{\theta},\xi)]+\frac{\text{Var}_{P_0}[H(\boldsymbol{\theta},\xi)]}{4\alpha}+\alpha \eta.
    \end{equation}

    Taking the derivative of~\cref{eq:30} with respect to $\alpha$ and setting it to 0, we get
    \begin{equation}
        \label{eq:31}
        \alpha^*=\frac{1}{2}\sqrt{\frac{\text{Var}_{P_0}[H(\boldsymbol{\theta},\xi)]}{\eta}}.
    \end{equation}

    Substituting $\alpha^*$ into~\cref{eq:28} yields
    \begin{equation}
        \label{eq:32}
        \begin{split}
            L^*(\xi)=1 
            +\sqrt{\frac{\eta}{\text{Var}_{P_0}[H(\boldsymbol{\theta},\xi)]}}[H(\boldsymbol{\theta},\xi)-\mathbb{E}_{P_0}[H(\boldsymbol{\theta},\xi)]].
        \end{split}
    \end{equation}

    Therefore, the closed-form optimal solution $P$ is
    \begin{equation}
        \label{eq:33}
        \begin{split}
            p^*(\xi)&=p_0(\xi)L^*(\xi) \\
            &=p_0(\xi) 
            +p_0(\xi)\sqrt{\frac{\eta}{\text{Var}_{P_0}[H(\boldsymbol{\theta},\xi)]}}[H(\boldsymbol{\theta},\xi)-\mathbb{E}_{P_0}[H(\boldsymbol{\theta},\xi)]].
        \end{split}
    \end{equation}

    Substituting $\alpha^*$ into~\cref{eq:29} yields
    \begin{equation}
        \label{eq:34}
        \begin{split}
        &\min\limits_{\alpha\geq 0}\mathbb{E}_{P_0}[H(\boldsymbol{\theta},\xi)]+\frac{\text{Var}_{P_0}[H(\boldsymbol{\theta},\xi)]}{4\alpha}+\alpha \eta \\
        =&\mathbb{E}_{P_0}[H(\boldsymbol{\theta},\xi)]+\sqrt{\eta \text{Var}_{P_0}[H(\boldsymbol{\theta},\xi)]}.
        \end{split}
    \end{equation}

    Here, we can transform the~\cref{eq:20} into a deterministic one
    \begin{equation}
        \label{eq:35}
        \min\limits_{\alpha\geq 0}\max\limits_{L(\xi)\in \mathbb{L}}l_0 (\alpha, L) = \mathbb{E}_{P_0}[H(\boldsymbol{\theta},\xi)]+\sqrt{\eta \text{Var}_{P_0}[H(\boldsymbol{\theta},\xi)]}.
    \end{equation}

    Therefore, we have
    \begin{equation}
        \label{eq:36}
        \max\limits_{P\in \Delta}\mathbb{E}_P[H(\boldsymbol{\theta},\xi)] = \mathbb{E}_{P_0}[H(\boldsymbol{\theta},\xi)]+\sqrt{\eta \text{Var}_{P_0}[H(\boldsymbol{\theta},\xi)]},
    \end{equation}
    which completes the proof.  
\end{proof}

\section{More details about fairness elasticity coefficient}\label{append:2}
In the field of economics, the elasticity coefficient primarily is employed as a measure of the sensitivity of one variable to changes in another. Nevertheless, it should be noted that this elasticity coefficient can assume negative values, which introduces a further layer of complexity into the analysis. To streamline this process, we propose an approach that involves exponentiating the increments of both variables, followed by calculating their ratio. This method not only preserves information about the relative rates of change between the two variables but also effectively circumvents the issue of negative values. In particular, when the rate of change of the numerator is negative, it indicates that the worst-class accuracy has not improved. Following exponentiation, the value becomes a positive number less than one, resulting in a reduction in the FEC. Conversely, if the rate of change of the denominator is negative, it indicates an increase in the average accuracy. As this term is in the denominator and will be a positive number less than one after exponentiation, it results in an increase in the FEC. The utilization of this processing method facilitates a more straightforward and discernible interpretation of the FEC. Consequently, the performance of an algorithm can be directly inferred from the magnitude of the FEC.

\section{Detailed information about the datasets}\label{append:3}
In this paper, we conduct comprehensive experiments on four benchmark datasets: CIFAR-10, CIFAR-100, SVHN, and STL10. CIFAR-10 comprises a total of 60,000 images, evenly distributed across 10 distinct classes. Each class comprises 6,000 images, with 5,000 images allocated for training and 1,000 images for testing. CIFAR-100, with a similar number of images (60,000), is divided into 100 classes, each with 600 images. The training and testing splits for each class are 500 and 100 images, respectively. SVHN comprises 10 distinct classes, with a training set of 73,257 images and a test set of 26,032 images. The images in CIFAR-10, CIFAR-100, and SVHN are sized at $3\times 32\times 32$. STL10 comprises 13,000 images across 10 classes, with each class having 500 images allocated for training and 800 for testing. The image size for STL10 is $3\times 96 \times 96$.

\section{More experiments}\label{append:4}

\begin{figure}[H]
    \centering
    \subfloat[AT and CODAT]{
    \includegraphics[scale=0.22]{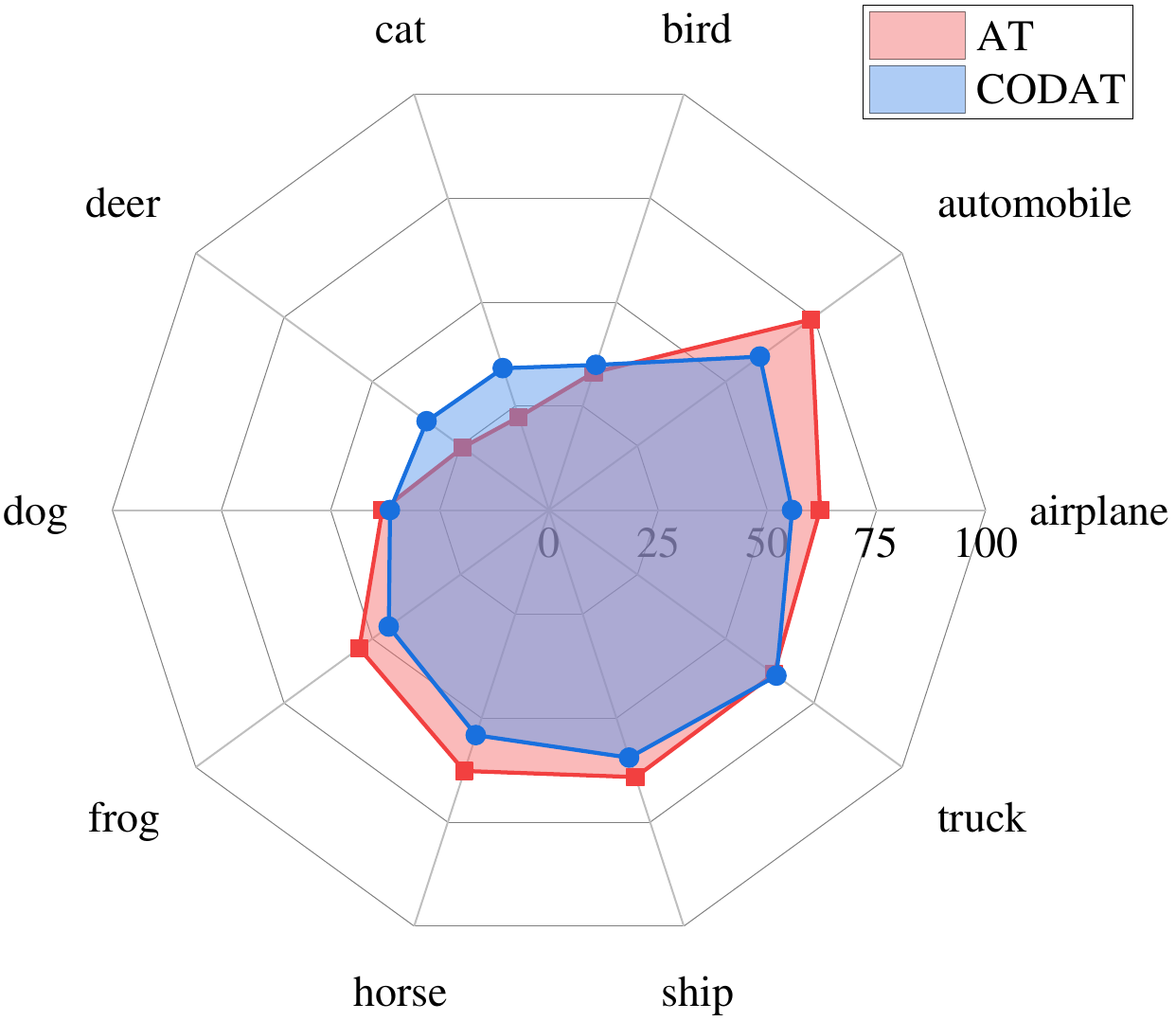}
    \label{fig:8-1}
    }
    \hfil
    \subfloat[TRADES and CODAT]{
    \includegraphics[scale=0.22]{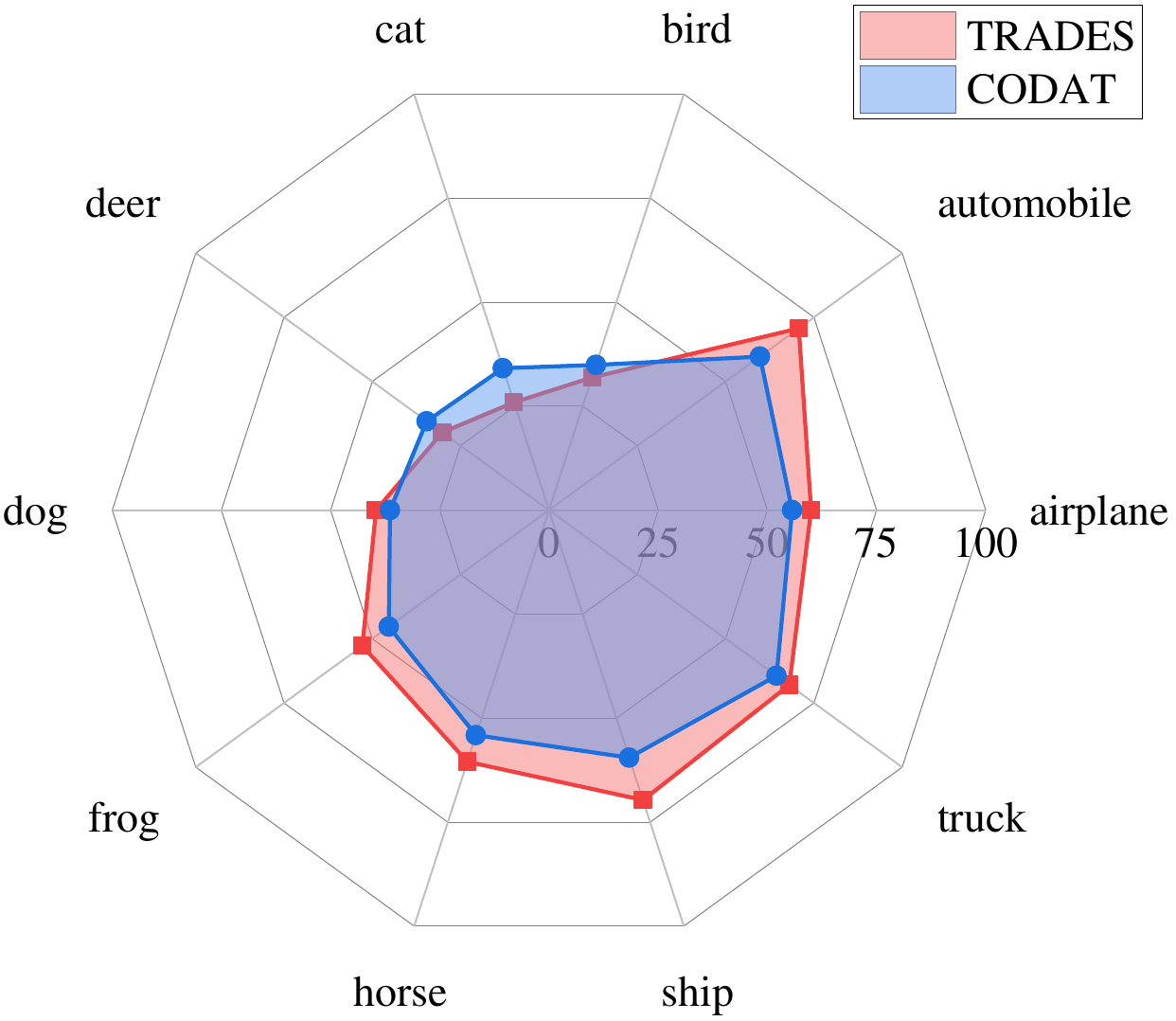}
    \label{fig:8-2}
    }
    \hfil
    \subfloat[FRL-RWRM and CODAT]{
    \includegraphics[scale=0.22]{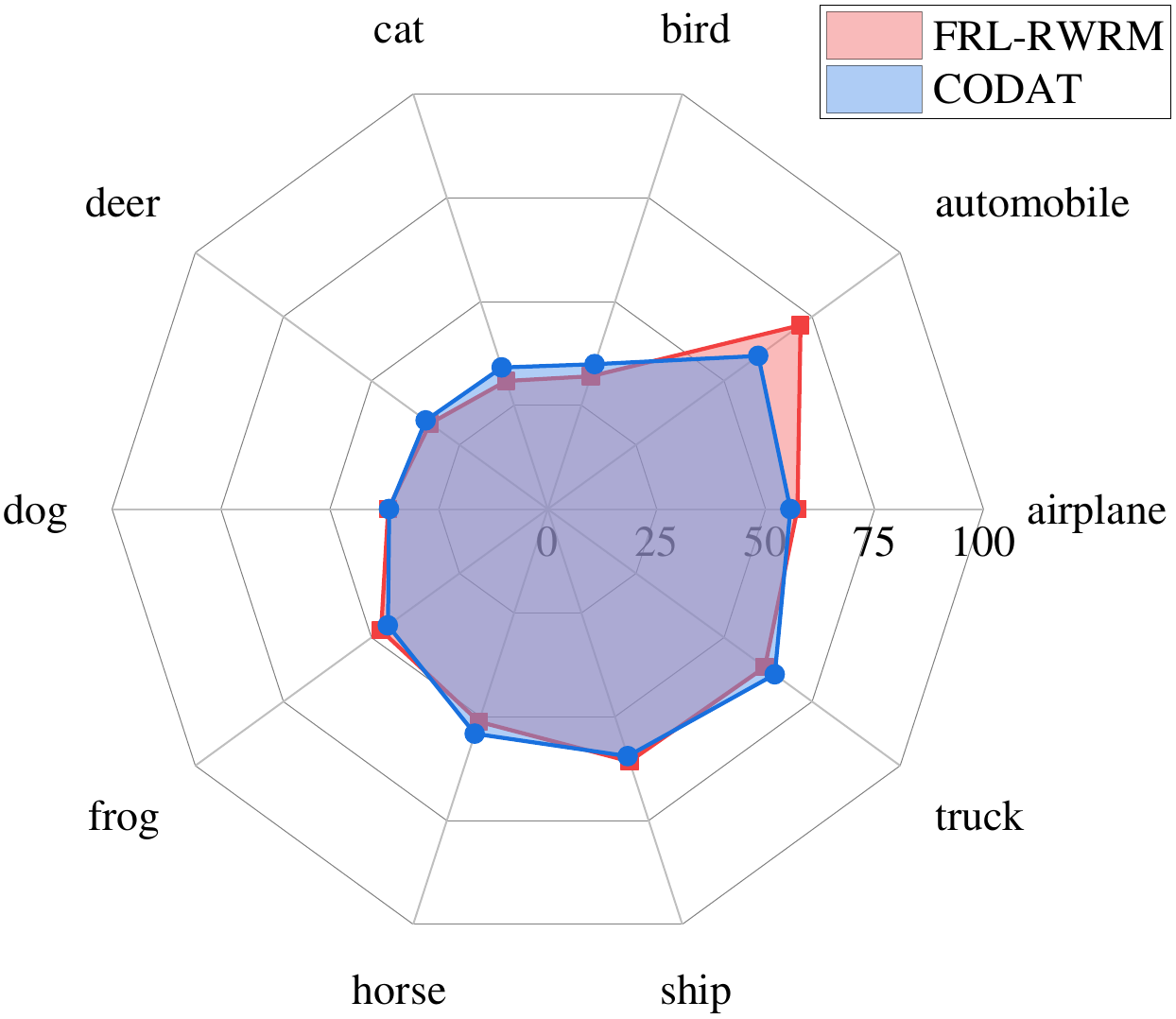}
    \label{fig:8-3}
    }
    \hfil

    \subfloat[BAT and CODAT]{
    \includegraphics[scale=0.22]{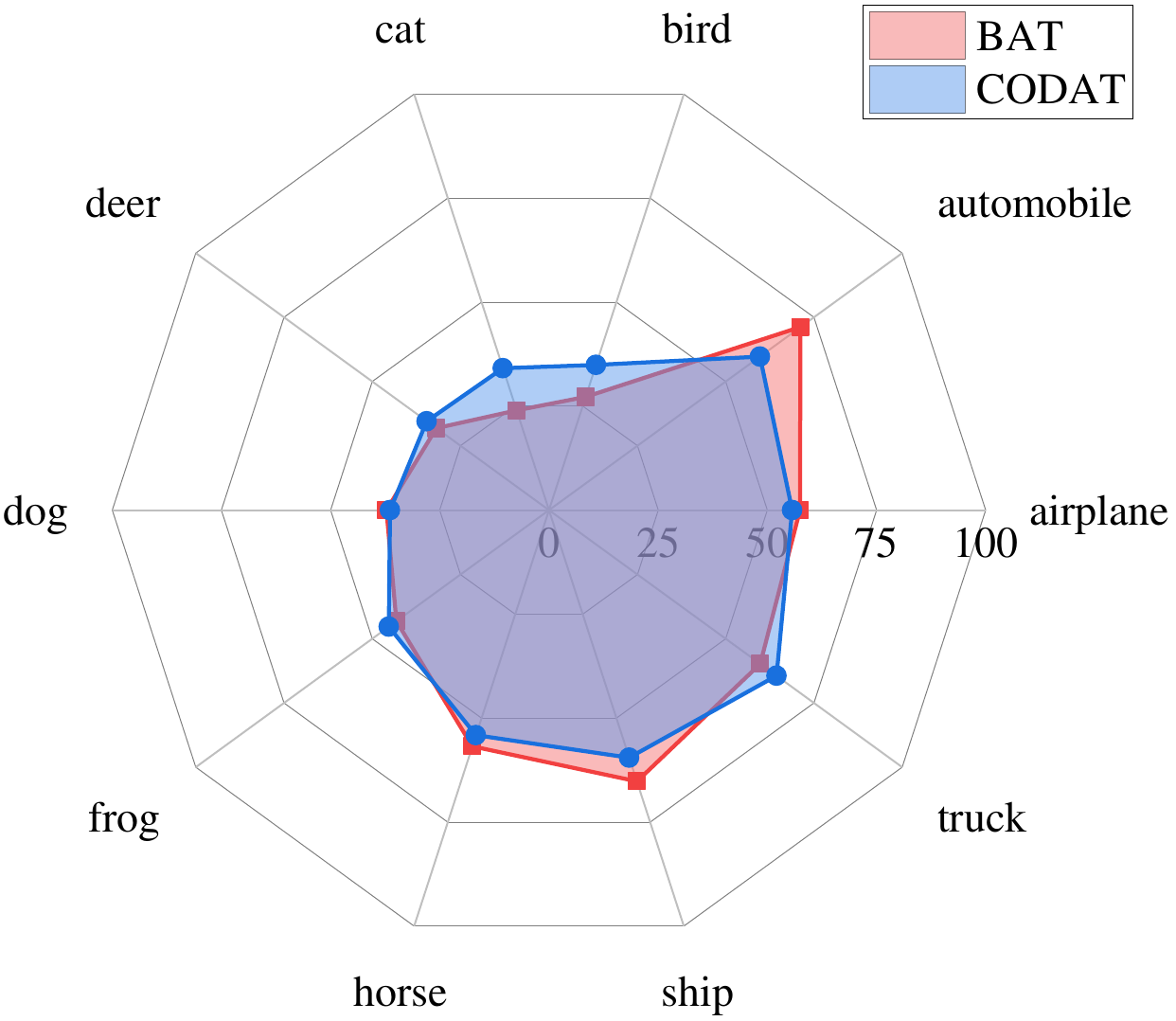}
    \label{fig:8-4}
    }
    \hfil\subfloat[CFOL and CODAT]{
    \includegraphics[scale=0.22]{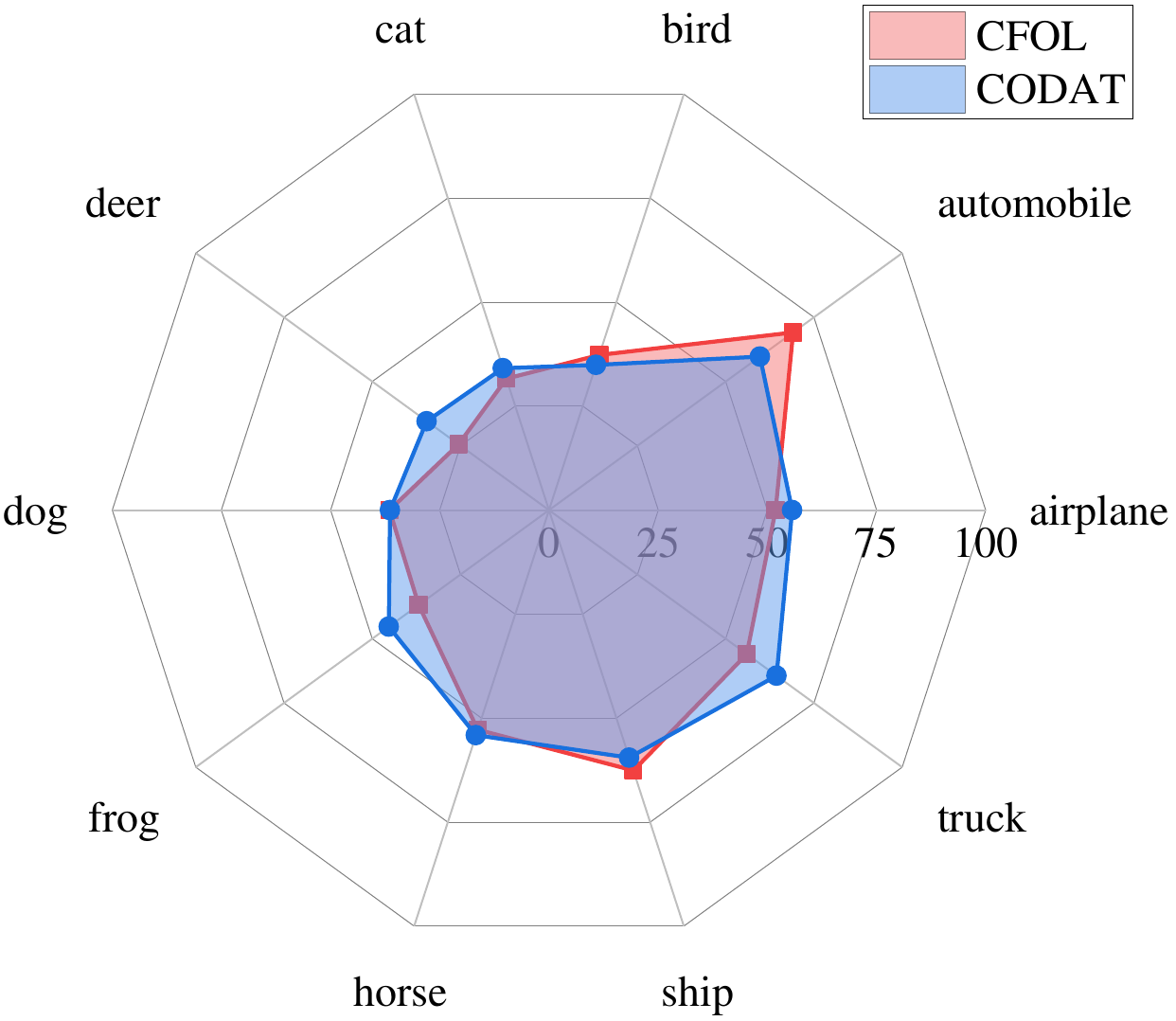}
    \label{fig:8-5}
    }
    \hfil\subfloat[WAT and CODAT]{
    \includegraphics[scale=0.22]{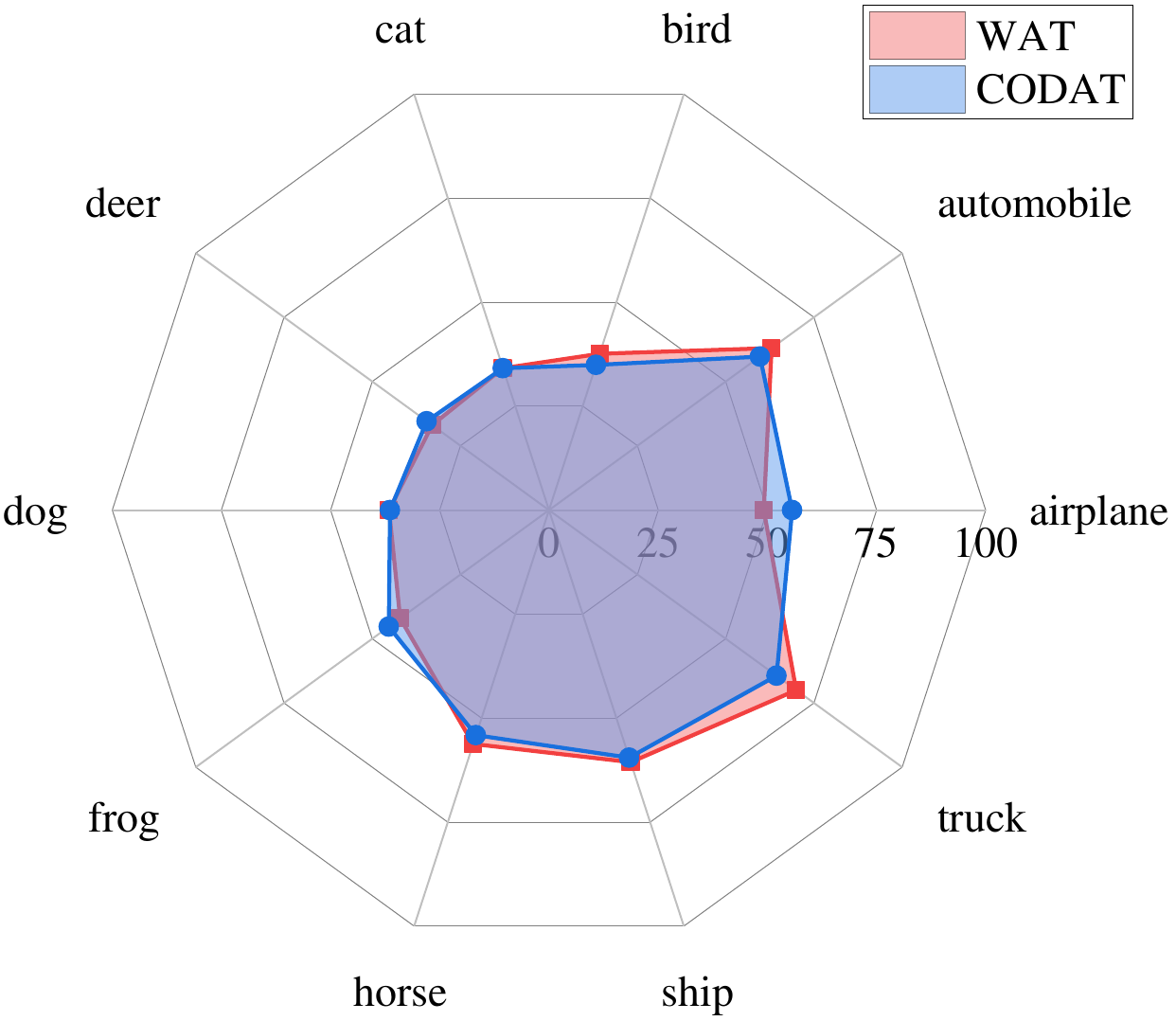}
    \label{fig:8-6}
    }
    \hfil
    \caption{Comparative analysis of class-wise robust accuracy for our method and baselines on CIFAR-10 using ResNet-18 under CW-30 attack.}
    \label{fig:8} 
\end{figure}

We perform a horizontal comparative analysis of class-wise robust accuracy with baseline methods on CIFAR-10 using ResNet-18 under CW-100 attack in~\cref{fig:8}. As shown in~\cref{fig:8-1} and~\cref{fig:8-2}, although our method exhibits a slight decrease in robust accuracy for some well-performing classes compared to traditional adversarial training methods such as AT and TRADES, it achieves significant improvements in the most vulnerable classes, specifically cat and deer, thereby enhancing robust fairness. The remaining figures in~\cref{fig:8} demonstrate that, with the exception of the class automobile, our method outperforms SOTA methods in almost all classes. This provides substantial evidence of the effectiveness of CODAT in improving robust fairness. 
\end{appendices}

\end{document}